\documentclass{article} 
\usepackage{iclr2027_conference,times}

\usepackage{amsmath,amsfonts,bm}

\def\eqref#1{equation~\ref{#1}}

\def\1{\bm{1}}

\DeclareMathAlphabet{\mathsfit}{\encodingdefault}{\sfdefault}{m}{sl}
\SetMathAlphabet{\mathsfit}{bold}{\encodingdefault}{\sfdefault}{bx}{n}

\newcommand{\R}{\mathbb{R}}

\usepackage{hyperref}
\usepackage{url}
\usepackage{amsthm}
\usepackage{amssymb}
\usepackage{mathtools}
\usepackage{enumitem}
\usepackage{booktabs}
\usepackage{algorithm}
\usepackage{algpseudocode}
\usepackage[table]{xcolor}
\newtheorem{definition}{Definition}
\newtheorem{proposition}{Proposition}

\newtheorem{theorem}{Theorem}
\newtheorem{corollary}{Corollary}
\definecolor{VisionHOPEGreen}{HTML}{E7F1E9}

\usepackage{multirow}
\title{\centering VisionHOPE: Visual Backbones as Self-Modifying Learning Systems}

\author{%
\parbox[t]{\dimexpr\textwidth-2\tabcolsep\relax}{
\centering
\normalfont
{\small
  \textbf{Siran Peng}\textsuperscript{2,3,$*$,$\ddagger$}\quad
  \textbf{Tianshuo Zhang}\textsuperscript{3,2,$*$}\quad
  \textbf{Tianyu Fu}\textsuperscript{1}\quad
  \textbf{Weisong Zhao}\textsuperscript{2}\quad
  \textbf{Haoyuan Zhang}\textsuperscript{3,2}
}\\
{\small
  \textbf{Jiankuo Zhao}\textsuperscript{2,3}\quad
  \textbf{Minghui Wu}\textsuperscript{1}\quad
  \textbf{Ping Jiang}\textsuperscript{1}\quad
  \textbf{Xiangyu Zhu}\textsuperscript{2,3}\quad
  \textbf{Chenxu Zhao}\textsuperscript{1,$\dagger$}\quad
  \textbf{Zhen Lei}\textsuperscript{2,3,4,$\dagger$} 
}\\
  \textnormal{\footnotesize \textsuperscript{1}Mininglamp Technology \quad \textsuperscript{2}MAIS, CASIA \quad \textsuperscript{3}SAI, UCAS \quad \textsuperscript{4}SCSE, FIE, M.U.S.T.} \\
  \texttt{\footnotesize \textnormal{\{pengsiran2023, zhangtianshuo2022, zhen.lei\}@ia.ac.cn, zhaochenxu@mininglamp.com}} \\
  \textnormal{\footnotesize \textsuperscript{$*$}Equal contribution. \quad \textsuperscript{$\dagger$}Corresponding authors.\quad \textsuperscript{$\ddagger$}Work done while interning at Mininglamp.}
}
}

\iclrfinalcopy 
\begin{document}

\maketitle
\pagestyle{plain}
\thispagestyle{plain}

\begin{abstract}
Visual backbones have evolved from Convolutional Neural Networks (CNNs) with local aggregation to Vision Transformers (ViTs) with global interactions, State-Space Models (SSMs) with input-dependent state transitions, and Test-Time Training (TTT) layers that adapt an inner learner while processing an image. Across this progression, visual computation has become increasingly adaptive to each input, yet the rules governing that adaptation remain largely prescribed by the trained backbone. We introduce VisionHOPE, the first generic visual backbone formulated as a self-modifying learning system, in which what the model remembers and how it learns co-evolve within an image. Building on the self-referential construction of Nested Learning (NL), VisionHOPE realizes this co-evolution through five coupled memories that store content, generate key and value representations, and govern learning rate and retention. These memories evolve jointly as visual context accumulates along each scan. However, directly applying the unconstrained self-referential update to a visual backbone leads to instability. We therefore derive a stability-matched step-size control scheme that combines a soft cap on self-referential injection with a spectral clamp on the retained memory transition, and prove that the resulting memory dynamics are non-expansive along each scan. For two-dimensional feature maps, we adapt NL's chunk formulation by aligning chunks with image rows and columns across four directional scans. The proposed VisionHOPE achieves competitive results on ImageNet-1K, COCO, and ADE20K, establishing self-modifying learning systems as a practical foundation for general-purpose visual backbones. The code is available at \href{https://github.com/PSRben/VisionHOPE}{this url}.
\end{abstract}

\section{Introduction}
\label{sec:introduction}
Visual backbone design has repeatedly advanced by changing how information is aggregated and propagated across spatial locations. Convolutional Neural Networks (CNNs) use shared local kernels, introducing locality and translation equivariance~\citep{NIPS2012_c399862d,He_2016_CVPR}, while modern ConvNets show that redesigned convolutional architectures remain highly competitive~\citep{Liu_2022_CVPR,Ding_2024_CVPR,Yu_2025_CVPR}. Vision Transformers (ViTs) instead use content-dependent self-attention, allowing each token to interact directly with global visual context~\citep{NIPS2017_3f5ee243,dosovitskiy2020image}. Subsequent work has introduced hierarchical architectures, refined spatial aggregation, improved attention efficiency, and revisited ViT block design~\citep{9710580,Fan_2024_CVPR,11298398,wang2026vit}. However, self-attention has quadratic complexity in token count, making high-resolution processing increasingly costly.

\begin{figure}[t]
\centering
\includegraphics[width=0.93\linewidth]{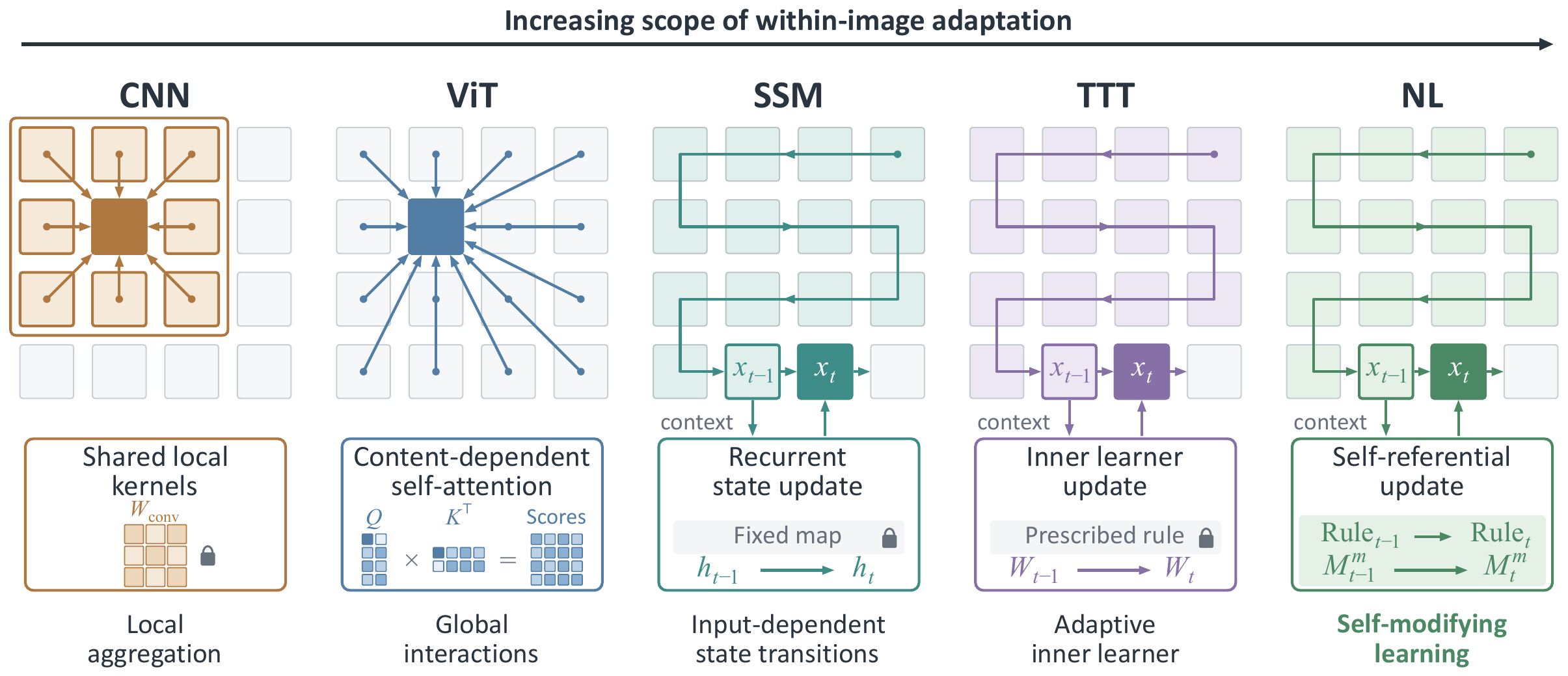}
\caption{Progression of within-image adaptation from local aggregation to self-modifying learning.}
\label{fig:visual_adaptation_progression}
\end{figure}

State-Space Models (SSMs) provide a recurrent alternative~\citep{gu2021efficiently}. Mamba makes its state transition input-dependent, allowing it to selectively propagate or forget information with linear sequence-length scaling~\citep{gu2023mamba}. Visual adaptations develop scanning strategies to accommodate the two-dimensional structure of images~\citep{pmlr-v235-zhu24f,NEURIPS2024_baa2da9a,huang2024localmamba} and incorporate attention modules to improve performance~\citep{Hatamizadeh_2025_CVPR}. Together, these advances establish SSMs as a strong family of visual backbones. Nevertheless, their transition-generation rule remains fixed: each token determines a transition, but the token-to-transition mapping itself does not evolve within the image.
Test-Time Training (TTT) goes further by using an inner learner as its recurrent state and updating it through self-supervised learning~\citep{pmlr-v267-sun25h}. In vision, ViT$^3$ studies full-image inner adaptation~\citep{Han_2026_CVPR}, whereas Vision-TTT adapts the learner recurrently along visual token sequences~\citep{kong2026vision}. These methods show that TTT layers offer an alternative to attention and conventional recurrent layers.

Figure~\ref{fig:visual_adaptation_progression} summarizes this progression: from CNNs to TTT, computation becomes increasingly adaptive to each input, yet the rules governing that adaptation remain prescribed by the trained backbone. This raises a question: \textbf{Can a visual backbone modify not only what it remembers, but also how it learns while processing an image?} We address this question through the self-referential principle of Nested Learning (NL)~\citep{NEURIPS2025_4309616a}. NL describes a model as a collection of interconnected learning processes, each compressing its context flow into an internal state. Its self-referential construction couples memories that store content, generate the key and value representations used for updates, and govern learning rate and retention. From this perspective, a backbone can become a learning system whose stored content and learning rule co-evolve with accumulating visual context.

Based on this principle, we introduce VisionHOPE, the first generic visual backbone formulated as a self-modifying learning system. VisionHOPE couples five types of within-image memory: content, key, value, learning rate, and retention. As visual context accumulates along each scan, the content memory co-evolves with the memories that determine its updates. VisionHOPE therefore modifies not only what it remembers but also how it learns. This distinguishes it from TTT-based visual backbones, which adapt an inner learner while leaving its representation maps and update rule largely outer-parameterized.
Directly applying NL's unconstrained self-referential update to a visual backbone, however, creates a stability barrier. Under this self-referential construction, the memories generate the quantities governing their own subsequent updates, creating a feedback loop in which expansive memory transitions can compound and destabilize training. We derive a stability-matched step-size control scheme that combines a soft cap on self-referential injection with a spectral clamp on the retained memory transition. We prove that these controls yield complementary bounds that jointly guarantee non-expansive memory dynamics along each scan.
To process two-dimensional feature maps, VisionHOPE adapts NL's chunk formulation~\citep{NEURIPS2025_4309616a} to four directional scans~\citep{NEURIPS2024_baa2da9a} by aligning chunks with image rows and columns. Forward and reverse row-major scans use row-aligned chunks, while forward and reverse column-major scans use column-aligned chunks. Each direction maintains independent states, and the directional outputs are fused channel-wise. The recurrent computation scales linearly with visual token count, while the chunk formulation facilitates parallel computation of token-dependent quantities. Our contributions are:
\begin{enumerate}[leftmargin=1.5em,labelsep=0.5em,itemsep=0.35em,topsep=0.35em]
\item We introduce VisionHOPE, the first generic visual backbone formulated as a self-modifying learning system. Its content, key, value, learning-rate, and retention memories co-evolve along visual scans, allowing the backbone to modify both what it remembers and how it learns.
\item We identify the unconstrained self-referential update as a stability barrier in vision and derive a stability-matched step-size control scheme. Combining a soft cap on self-referential injection with a spectral clamp on the retained transition, it guarantees non-expansive memory dynamics.
\item We adapt NL's chunk formulation to four directional scans and instantiate the resulting VisionHOPE operator in hierarchical and plain backbones. Across model scales, VisionHOPE achieves competitive results on ImageNet-1K~\citep{deng2009imagenet}, COCO~\citep{lin2014microsoft}, and ADE20K~\citep{Zhou_2017_CVPR}, demonstrating its potential as a general-purpose visual backbone.
\end{enumerate}

\section{Methodology}
\label{sec:method}
VisionHOPE instantiates NL's self-referential principle as a stabilized Self-Referential Nested Learning (SRNL) module. We first review the theoretical foundations of SRNL, then introduce the stability-matched step-size control and prove non-expansion for the resulting token-wise and chunk-wise recurrences. Finally, we apply independent SRNL instances along four directional scans with spatially aligned chunks and embed the resulting VisionHOPE operator in a standard residual block.

\subsection{Nested Learning Foundations}
\label{sec:nl_preliminaries}
Nested Learning (NL) describes a model as a collection of interconnected learning processes, each compressing its own context flow into an internal state~\citep{NEURIPS2025_4309616a}. To establish the theoretical foundations of VisionHOPE, we first review associative memory and Delta Gradient Descent (DGD), then describe NL's self-referential construction and the resulting coupled memory updates. Finally, we present the chunk-wise formulation for efficient computation of these updates.

\subsubsection{Associative Memory and Delta Gradient Descent}
\label{sec:dgd}

\paragraph{Associative memory.}
Let $\mathcal K=\{k_i\}_{i=1}^{n}$ and $\mathcal V=\{v_i\}_{i=1}^{n}$ denote paired keys and values, with $k_i\in\R^{d_k}$ and $v_i\in\R^{d_v}$. Here, $n$ is the number of associations, while $d_k$ and $d_v$ are the respective key and value dimensions. Following NL, an associative memory is a parameterized map $M:\R^{d_k}\rightarrow\R^{d_v}$ that compresses these associations into its parameters. Given an internal objective $\widetilde{\mathcal L}$ that measures how well $M$ maps each key to its paired value, learning the memory is formulated as follows:
\begin{equation}
M^\star = \arg\min_M\widetilde{\mathcal L}\left(M(\mathcal K);\mathcal V\right),
\label{eq:nl_associative_memory}
\end{equation}
where $M(\mathcal K)=\{M(k_i)\}_{i=1}^{n}$ denotes predictions for all keys, and $M^\star$ is the optimized memory. The parameters of $M$ form its memory state, whereas optimizing this objective is the associated learning process. The keys and values are not restricted to input tokens and may represent data samples, gradients, intermediate representations, subsequences, or other elements of a context flow.

\paragraph{Gradient descent as associative memory.}
NL observes that gradient descent itself can be interpreted as an associative-memory process. To instantiate the general memory $M$ in Equation~(\ref{eq:nl_associative_memory}), set $d_k=d_v=d$ and consider a linear associative memory $M_W:\R^d\rightarrow\R^d$. The memory is parameterized by $W\in\R^{d\times d}$ and maps a key $k$ as $M_W(k)=Wk$. At online step $t$ in the context flow, $W_{t-1}$ denotes the memory state after the preceding $t-1$ elements, and $k_t\in\R^d$ denotes the current key. The memory response is $z_t=M_{W_{t-1}}(k_t)=W_{t-1}k_t\in\R^d$. Let $\mathcal J_t:\R^d\rightarrow\R$ be a differentiable local objective on the memory response. Its gradient at $z_t$ defines the output-space learning signal $g_t=\left.\nabla_z\mathcal J_t(z)\right|_{z=z_t}\in\R^d$. By the chain rule, the gradient with respect to the memory parameters is
\begin{equation}
\left.\nabla_W\mathcal J_t(Wk_t)\right|_{W=W_{t-1}}=g_tk_t^\top.
\label{eq:nl_matrix_gradient}
\end{equation}
Here, $\nabla_W$ denotes differentiation with respect to $W$, $(\cdot)^\top$ denotes transpose, and $g_tk_t^\top\in\R^{d\times d}$ is a rank-one outer product. Ordinary gradient descent then updates the memory as
\begin{equation}
W_t^{\mathrm{GD}}=W_{t-1}-\eta_tg_tk_t^\top,\qquad \eta_t>0,
\label{eq:nl_gd_update}
\end{equation}
where $W_t^{\mathrm{GD}}$ is the updated memory and $\eta_t$ is the learning rate at the current step. The rank-one term $-\eta_tg_tk_t^\top$ admits an associative-memory interpretation: $k_t$ specifies the write address, while $-g_t$ provides the write signal. NL equivalently formulates this update as the solution to a proximal problem that combines a dot-product mapping objective with a quadratic penalty on changes to the previous memory state. Appendix~\ref{app:gd_associative_memory} gives the derivation. In this view, $W$ serves as the memory state, while gradient descent is the learning process that compresses input-dependent learning signals into it. The dot-product formulation yields an additive rank-one write with no additional correction based on the current response $W_{t-1}k_t$. NL argues that such an additional correction is useful for context flows with highly correlated elements, motivating DGD as an alternative learning rule.

\paragraph{Delta Gradient Descent.}
NL introduces DGD to incorporate an explicit correction based on the memory response at the current key. It uses the write signal $u_t=-g_t$ as the regression target and replaces the dot-product mapping objective with an auxiliary $L_2$ regression objective. We denote the proximal parameter by $\lambda_t>0$ to distinguish it from the effective step used in the resulting recurrence. With $k_t$ and $u_t$ held fixed, the updated memory $W_t^{\mathrm{DGD}}$ is defined as follows:
\begin{equation}
W_t^{\mathrm{DGD}}=\arg\min_{W\in\R^{d\times d}}\left[\frac{1}{2}\left\lVert Wk_t-u_t\right\rVert_2^2+\frac{1}{2\lambda_t}\left\lVert W-W_{t-1}\right\rVert_F^2\right],\qquad \lambda_t>0.
\label{eq:nl_dgd_proximal_objective}
\end{equation}
The first term measures the squared regression error between the candidate response $Wk_t$ and the write target $u_t$ at the current key, while the second penalizes changes from the previous memory state $W_{t-1}$. Solving this proximal problem gives the following effective DGD step size
\begin{equation}
\eta_t=\frac{\lambda_t}{1+\lambda_t\lVert k_t\rVert_2^2},
\label{eq:nl_dgd_effective_step}
\end{equation}
which reduces to $\eta_t=\lambda_t/(1+\lambda_t)$ for unit-norm keys. The resulting closed-form recurrence is
\begin{equation}
W_t^{\mathrm{DGD}}=W_{t-1}-\eta_t\left(W_{t-1}k_t-u_t\right)k_t^\top=W_{t-1}\left(I_d-\eta_tk_tk_t^\top\right)-\eta_tg_tk_t^\top.
\label{eq:nl_dgd_update}
\end{equation}
Here, $I_d\in\R^{d\times d}$ is the identity matrix. Appendix~\ref{app:dgd_proximal_derivation} provides the derivation. Compared with ordinary gradient descent, DGD adds the key-dependent correction $-\eta_tW_{t-1}k_tk_t^\top$. Equivalently, it writes the residual $u_t-W_{t-1}k_t$, explicitly correcting the current write signal using what the memory already stores at the current key. This provides a mechanism for revising existing associations as successive, correlated context elements arrive. The self-referential construction of NL augments this DGD recurrence with a context-dependent retention factor denoted by $\alpha_t\in[0,1)$:
\begin{equation}
W_t=W_{t-1}\left(\alpha_tI_d-\eta_tk_tk_t^\top\right)-\eta_tg_tk_t^\top.
\label{eq:dgd_with_retention}
\end{equation}
Here, $\alpha_t$ controls how much of the previous memory is retained at each step of the recurrence, and $\eta_t$ scales both the key-dependent correction and the current gradient write. This retained DGD recurrence is the inner learning rule applied to the coupled memories introduced next.

\subsubsection{Self-Referential Nested Learning}
\label{sec:self_referential_learning}
The retained DGD recurrence in Equation~(\ref{eq:dgd_with_retention}) specifies how a memory is updated at each online step given its current key, learning signal, step size, and retention factor. NL makes this process self-referential by using evolving memories to generate the update quantities and allowing each memory to produce its own target. NL permits arbitrary memory architectures. In the original HOPE architecture in NL, these memories are instantiated as Multi-Layer Perceptrons (MLPs), and their mappings are learned with an $L_2$ regression objective. The corresponding $L_2$ recurrence, however, is derived explicitly only for matrix-valued linear memories. VisionHOPE accordingly adopts this linear $L_2$ recurrence for memory updates, yielding the five-memory SRNL module used throughout the model. For an ordered visual context, let $x_t\in\R^d$ denote the token at step $t$, where $d$ is the input dimension of an SRNL instance. The self-referential state contains the following five memories:
\begin{equation}
\mathcal S_t=\left\{M^m_t,M^k_t,M^v_t,m^\eta_t,m^\alpha_t\right\}.
\label{eq:nl_self_referential_state}
\end{equation}
Here, $M^m_t\in\R^{d\times d}$ stores content, $M^k_t,M^v_t\in\R^{d\times d}$ generate key and value representations, and $m^\eta_t,m^\alpha_t\in\R^{1\times d}$ govern learning rate and retention. For compactness, $M^\square_t$ denotes any of these states for $\square\in\{m,k,v,\eta,\alpha\}$, with $M^\eta_t=m^\eta_t$ and $M^\alpha_t=m^\alpha_t$. At step $t$, the update quantities are generated from the preceding memory states using the current token $x_t$:
\begin{equation}
k_t=M^k_{t-1}x_t,\qquad v_t=M^v_{t-1}x_t,\qquad \eta_t=\phi_\eta\left(m^\eta_{t-1}x_t\right),\qquad \alpha_t=\phi_\alpha\left(m^\alpha_{t-1}x_t\right).
\label{eq:nl_self_referential_quantities}
\end{equation}
Here, $k_t,v_t\in\R^d$ are the key and value representations. The functions $\phi_\eta:\R\rightarrow\R_{>0}$ and $\phi_\alpha:\R\rightarrow[0,1)$ map the scalar memory outputs to a positive DGD step $\eta_t$ and a retention factor $\alpha_t$. Their forms are specified in Appendix~\ref{app:model_configurations}. The learning-rate memory directly parameterizes $\eta_t$, while $\lambda_t$ is used only in the preceding proximal derivation. Self-reference enters through each memory's update target. Rather than sharing $v_t$ as the target, memory $\square$ transforms it using its current state:
\begin{equation}
\widehat v^\square_t=M^\square_{t-1}v_t,\qquad \square\in\{m,k,v,\eta,\alpha\}.
\label{eq:nl_self_generated_target}
\end{equation}
Each memory uses its self-generated target $\widehat v^\square_t$ to define a learning signal at the current key $k_t$.  Following the output-space formulation in Section~\ref{sec:dgd}, we define the local $L_2$ regression objective
\begin{equation}
\mathcal J^\square_t(z)=\frac{1}{2}\left\lVert z-\widehat v^\square_t\right\rVert_2^2,
\label{eq:nl_self_referential_objective}
\end{equation}
where $z$ denotes the memory response. Both the response and target are vectors for matrix memories and scalars for row-vector memories. Holding the target fixed and evaluating the gradient with respect to $z$ at the current response gives the output-space learning signal for each memory
\begin{equation}
g^\square_t=\left.\nabla_z\mathcal J^\square_t(z)\right|_{z=M^\square_{t-1}k_t}
=M^\square_{t-1}k_t-\widehat v^\square_t.
\label{eq:nl_self_referential_signal}
\end{equation}
Substituting this signal into retained DGD gives the Self-Referential DGD (SR-DGD) update
\begin{equation}
M^\square_t=M^\square_{t-1}\left(\alpha_tI_d-\eta_tk_tk_t^\top\right)-\eta_tg^\square_tk_t^\top.
\label{eq:nl_online_srdgd}
\end{equation}
The same expression applies to row-vector memories as $1\times d$ linear maps. Because all five states generate quantities used to determine their subsequent updates, stored content, update representations, and learning dynamics evolve as a coupled system. The recurrence is explicit and differentiable, so outer gradients can pass through the memory trajectory without an iterative inner solver.

The query projection remains outside the self-referential update loop. Following NL's construction, an outer-parameterized projection generates the query to read the output from the content memory:
\begin{equation}
q_t=W_qx_t,\qquad y_t=M^m_{t-1}q_t.
\label{eq:nl_nonadaptive_query}
\end{equation}
Here, $W_q\in\R^{d\times d}$ is the query projection, while $q_t,y_t\in\R^d$ are the query and output. The projection is learned through outer training but remains fixed during within-context memory evolution.

\subsubsection{Chunk-Wise Linear Recurrence}
\label{sec:nl_chunk_formulation}
The fully token-indexed SRNL recurrence in Equation~(\ref{eq:nl_online_srdgd}) must update its state at step $t$ before generating the quantities for step $t+1$, serializing the state updates and the five memory mappings across tokens. NL reduces this cost through chunking. Within each chunk, fixed boundary states generate token-dependent quantities in parallel, while the SR-DGD state updates accumulate in token order. Chunking refreshes the states used to generate update quantities only at chunk boundaries. We partition the ordered context into chunks of fixed length $C$. For token $t$, the state of memory $\square\in\{m,k,v,\eta,\alpha\}$ at the beginning of its chunk is $M^\square_{C\times\left\lfloor\frac{t-1}{C}\right\rfloor}$. These boundary states generate
\begin{equation}
\small
k_t=M^k_{C\times\left\lfloor\frac{t-1}{C}\right\rfloor}x_t,\quad v_t=M^v_{C\times\left\lfloor\frac{t-1}{C}\right\rfloor}x_t,\quad \eta_t=\phi_\eta\!\left(m^\eta_{C\times\left\lfloor\frac{t-1}{C}\right\rfloor}x_t\right),\quad \alpha_t=\phi_\alpha\!\left(m^\alpha_{C\times\left\lfloor\frac{t-1}{C}\right\rfloor}x_t\right).
\label{eq:nl_chunk_generated_quantities}
\end{equation}
With $q_t=W_qx_t$, the content readout, self-generated target, and output-space learning signal are
\begin{equation}
y_t=M^m_{C\times\left\lfloor\frac{t-1}{C}\right\rfloor}q_t,\qquad \widehat v^\square_t=M^\square_{C\times\left\lfloor\frac{t-1}{C}\right\rfloor}v_t,\qquad g^\square_t=M^\square_{C\times\left\lfloor\frac{t-1}{C}\right\rfloor}k_t-\widehat v^\square_t.
\label{eq:nl_chunk_target_signal}
\end{equation}
Because the boundary states remain fixed throughout the chunk, the quantities in Equations~(\ref{eq:nl_chunk_generated_quantities}) and~(\ref{eq:nl_chunk_target_signal}) can be computed in parallel over the tokens within that chunk. The working memory states are then updated in token order using Equation~(\ref{eq:nl_online_srdgd}) with these update quantities. At the end of each chunk, the accumulated memory states become the boundary states used for the next chunk. 

\subsection{Stability-Matched Step-Size Control}
\label{sec:soft_projection}
The unconstrained SR-DGD update couples every memory to quantities generated by the evolving system itself within each image. To isolate the resulting feedback, we first analyze the fully token-indexed recurrence. The resulting step-size control is applied at each token and has the same form under chunking. Let $\delta_t=k_t-v_t\in\R^d$ denote the key-value discrepancy. Equations~(\ref{eq:nl_self_generated_target}) and~(\ref{eq:nl_self_referential_signal}) then give $g_t^\square=M_{t-1}^\square\delta_t$. Substituting this relation into Equation~(\ref{eq:nl_online_srdgd}) yields the recurrence:
\begin{equation}
M_t^\square=M_{t-1}^\square\left(\alpha_tI_d-\eta_tk_tk_t^\top-\eta_t\delta_tk_t^\top\right).
\label{eq:visionhope_unconstrained_transition}
\end{equation}
The matrix multiplying $M_{t-1}^\square$ is the one-step memory transition. For $k_t\neq0$, evaluating this transition along the unit direction of $k_t$ and applying the reverse triangle inequality gives
\begin{equation}
\left\lVert\alpha_tI_d-\eta_tk_tk_t^\top-\eta_t\delta_tk_t^\top\right\rVert_2=\left\lVert\alpha_tI_d-\eta_t(k_t+\delta_t)k_t^\top\right\rVert_2\geq\eta_t\lVert k_t\rVert_2\lVert k_t+\delta_t\rVert_2-\alpha_t.
\label{eq:visionhope_unconstrained_expansion_bound}
\end{equation}
Thus, $\eta_t\lVert k_t\rVert_2\lVert k_t+\delta_t\rVert_2>1+\alpha_t$ is sufficient for an expansive transition, and the recurrence provides no general non-expansion guarantee. Because the updated memories generate the quantities governing subsequent updates, such amplification can feed back and compound along the context. Under chunking, the same feedback passes through the boundary states propagated between chunks. To guarantee non-expansion, we bound the spectral norm of the retained memory transition $\alpha_tI_d-\eta_tk_tk_t^\top$ by $\alpha_t$ and keep that of the self-referential injection $-\eta_t\delta_tk_t^\top$ below $1-\alpha_t$. We implement these complementary constraints with a soft injection cap followed by a spectral clamp.

\begin{definition}[Stability-matched step-size control]
\label{def:stability_matched_projection}
VisionHOPE applies stabilized $L_2$ normalization to each key, ensuring $\lVert k_t\rVert_2\leq1$. Let $\epsilon>0$ be a constant for numerical stability. For the raw step $\eta_t>0$ and retention factor $\alpha_t\in[0,1)$ generated by their memories, define the soft injection cap:
\begin{equation}
r_t=\sqrt{\lVert\delta_t\rVert_2^2+\epsilon^2}+\epsilon,\qquad \eta_t^{\mathrm{inj}}=\frac{1-\alpha_t}{r_t},\qquad \bar\eta_t=\eta_t^{\mathrm{inj}}\left(1-\exp\left(-\frac{\eta_t}{\eta_t^{\mathrm{inj}}}\right)\right),
\label{eq:visionhope_soft_projection}
\end{equation}
where $r_t$ is the stabilized magnitude of the key-value discrepancy, $\eta_t^{\mathrm{inj}}$ is the step-size limit for self-referential injection, and $\bar\eta_t$ is the smoothly capped candidate step. We then define the spectral clamp, which enforces the retained-transition bound and produces the final executed step $\widetilde\eta_t$:
\begin{equation}
\eta_t^{\mathrm{spec}}=\frac{2\alpha_t}{\lVert k_t\rVert_2^2},\qquad \widetilde\eta_t=\min\!\left(\bar\eta_t,\eta_t^{\mathrm{spec}}\right),
\label{eq:visionhope_spectral_clamp}
\end{equation}
where $\eta_t^{\mathrm{spec}}$ is the spectral limit for the retained memory transition and is set to $+\infty$ when $k_t=0$. The soft injection cap smoothly limits self-referential injection while preserving small raw steps to first order. The spectral clamp leaves the candidate unchanged unless it exceeds $\eta_t^{\mathrm{spec}}$.
\end{definition}

Replacing $\eta_t$ in Equation~(\ref{eq:visionhope_unconstrained_transition}) with $\widetilde\eta_t$ gives
\begin{equation}
M_t^\square=M_{t-1}^\square T_t,\qquad T_t=A_t+B_t,\qquad A_t=\alpha_tI_d-\widetilde\eta_tk_tk_t^\top,\qquad B_t=-\widetilde\eta_t\delta_tk_t^\top.
\label{eq:visionhope_projected_transition}
\end{equation}
Here, $A_t$ is the retained memory transition, $B_t$ is the self-referential injection operator, and $T_t$ is the complete controlled transition governing the memory update at each token.

\begin{proposition}[Complementary operator bounds]
\label{prop:visionhope_operator_bounds}
Under Definition~\ref{def:stability_matched_projection}, the two transition components satisfy the following complementary operator bounds:
\begin{equation}
\lVert A_t\rVert_2\leq\alpha_t,\qquad \lVert B_t\rVert_2<1-\alpha_t.
\label{eq:visionhope_operator_bounds}
\end{equation}
\end{proposition}
Appendix~\ref{app:projection_properties} provides the proof. The first bound limits the spectral norm of the retained memory transition to $\alpha_t$. The second scales the admissible self-referential injection with the contraction margin left by retention, allowing less additional feedback as $\alpha_t$ approaches one. Together, these bounds control both potential sources of one-step amplification in the complete transition.

\begin{corollary}[Token-wise non-expansion]
\label{cor:visionhope_token_nonexpansion}
For the fully token-indexed recurrence, the following transition and state bounds hold for every memory $\square\in\{m,k,v,\eta,\alpha\}$ and token $t$:
\begin{equation}
\lVert T_t\rVert_2<1,\qquad \lVert M_t^\square\rVert_F\leq\lVert M_{t-1}^\square\rVert_F.
\label{eq:visionhope_token_nonexpansion}
\end{equation}
\end{corollary}
The corollary establishes non-expansion of memory norms along the token-wise recurrence and is proved in Appendix~\ref{app:projection_properties}. The same step-size control gives the complementary operator bounds at each token for quantities generated from chunk boundary states. Appendix~\ref{app:chunk_stability} derives a shared-gain representation and proves non-expansion relative to the boundary states of each chunk.

\begin{figure}[t]
\centering
\includegraphics[width=0.99\linewidth]{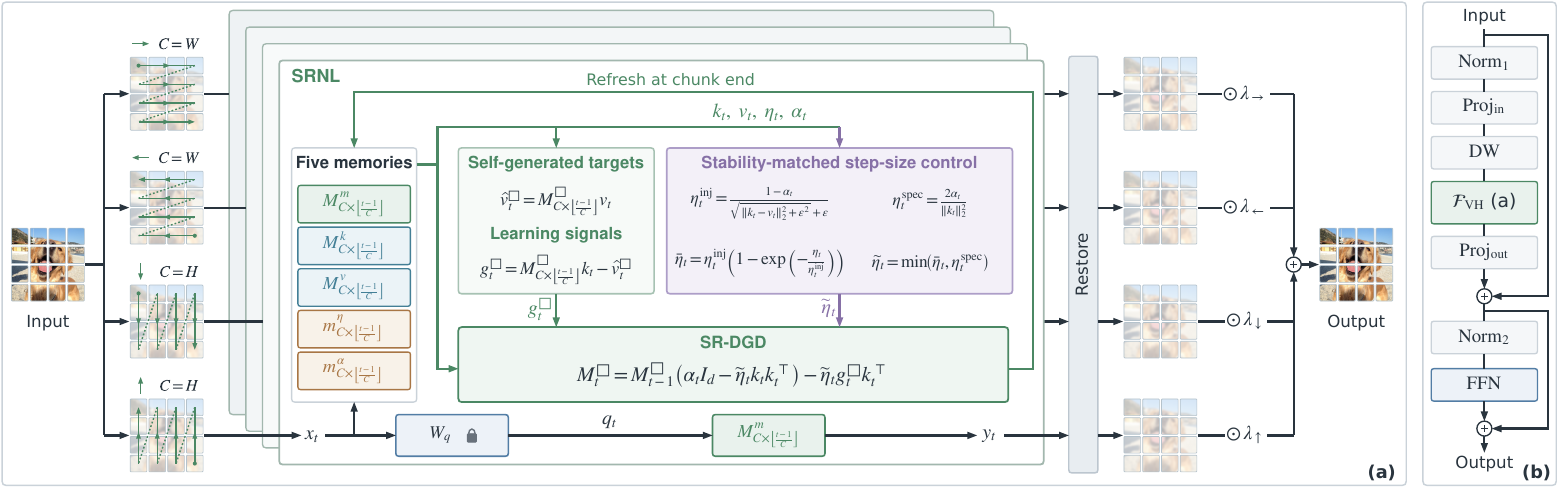}
\caption{(a) The VisionHOPE operator applies SRNL along four directional scans with spatially aligned chunks, followed by spatial restoration and channel-wise fusion. Five coupled memories evolve through SR-DGD with step-size control. (b) The VisionHOPE block integrates the operator with projections, depthwise convolution, and an FFN in two pre-normalized residual branches.}
\label{fig:visionhope_architecture}
\end{figure}

\subsection{VisionHOPE Architecture}
\label{sec:visionhope_architecture}
\paragraph{Four-directional VisionHOPE operator.}
Sections~\ref{sec:nl_preliminaries} and~\ref{sec:soft_projection} define the SRNL module over an ordered visual context and equip its recurrence with the stability-matched step-size control. The VisionHOPE operator extends SRNL to two-dimensional feature maps through four directional scans~\citep{NEURIPS2024_baa2da9a}, followed by spatial restoration and channel-wise fusion (Figure~\ref{fig:visionhope_architecture}(a)). Along each route, chunk boundaries are aligned with complete rows or columns. Let $Z\in\R^{B\times D_m\times H\times W}$ denote the internal feature map for scanning, where $B$ is the batch size, $D_m$ is the internal width, and $H$ and $W$ are the spatial dimensions. With $N=HW$, VisionHOPE serializes $Z$ as
\begin{equation}
\mathcal D=\{\rightarrow,\leftarrow,\downarrow,\uparrow\},\qquad Z_r=\mathcal P_r(Z)\in\R^{B\times N\times D_m},\qquad r\in\mathcal D,
\label{eq:visionhope_directional_serialization}
\end{equation}
where $\mathcal P_r$ serializes the feature grid in either direction under row-major or column-major ordering. Each direction runs an independent SRNL instance with the five-memory state defined in Equation~(\ref{eq:nl_self_referential_state}), initialized from learned values at the beginning of each image. Chunks follow the rows or columns traversed by each route. Let $C_r$ denote the chunk length for direction $r$. VisionHOPE sets
\begin{equation}
C_{\rightarrow}=C_{\leftarrow}=W,\qquad C_{\downarrow}=C_{\uparrow}=H.
\label{eq:visionhope_spatial_chunk_lengths}
\end{equation}
The row-major routes contain $H$ row-aligned chunks, while the column-major routes contain $W$ column-aligned chunks. Within each SRNL, the chunk-wise recurrence in Section~\ref{sec:nl_chunk_formulation} is computed using the executed step from Equation~(\ref{eq:visionhope_spectral_clamp}), preserving parallel generation of token-dependent quantities within each chunk. Let $Y_r\in\R^{B\times N\times D_m}$ denote the directional output for route $r$. The four sequences are restored to the common feature grid and fused using learned channel-wise scales:
\begin{equation}
Y=\sum_{r\in\mathcal D}\lambda_r\odot\mathcal P_r^{-1}(Y_r),\qquad \lambda_r\in\R^{D_m},
\label{eq:visionhope_directional_fusion}
\end{equation}
where $\mathcal P_r^{-1}$ reverses the directional serialization, $\lambda_r$ weights the output of route $r$ channel-wise, and $\odot$ denotes element-wise multiplication. These operations define the VisionHOPE operator $\mathcal F_{\mathrm{VH}}(Z)=Y$, which preserves the feature-map shape and scales linearly with token count $N$.

\paragraph{Block design.}
A VisionHOPE block combines the VisionHOPE operator with input and output projections, local depthwise convolution, and a Feed-Forward Network (FFN) in two pre-normalized residual branches (Figure~\ref{fig:visionhope_architecture}(b)). For a block input $X\in\R^{B\times D\times H\times W}$, the structural form is
\begin{equation}
\footnotesize
Z=\operatorname{DW}\!\left(\operatorname{Proj}_{\mathrm{in}}\!\left(\operatorname{Norm}_1(X)\right)\right),\quad X'=X+\operatorname{Proj}_{\mathrm{out}}\!\left(\mathcal F_{\mathrm{VH}}(Z)\right),\quad X_{\mathrm{out}}=X'+\operatorname{FFN}\!\left(\operatorname{Norm}_2(X')\right),
\label{eq:visionhope_simplified_block}
\end{equation}
where $\operatorname{Norm}_1$ and $\operatorname{Norm}_2$ denote normalization layers, $\operatorname{Proj}_{\mathrm{in}}$ and $\operatorname{Proj}_{\mathrm{out}}$ are the pointwise input and output projections that map channels from $D$ to $D_m$ and from $D_m$ to $D$, respectively, and $\operatorname{DW}$ denotes depthwise convolution. Appendix~\ref{app:visionhope_architecture_details} provides further implementation details.

\section{Experiments}
\label{sec:experiments}
We evaluate VisionHOPE on ImageNet-1K classification~\citep{deng2009imagenet}, COCO object detection and instance segmentation~\citep{lin2014microsoft}, and ADE20K semantic segmentation~\citep{Zhou_2017_CVPR}. We further examine computational efficiency and key design choices. Model configurations, training protocols, and visual analysis are provided in Appendices~\ref{app:model_configurations},~\ref{app:implementation_details}, and~\ref{app:visual_analysis}, respectively.

\subsection{Image Classification}
We train VisionHOPE from scratch on ImageNet-1K, which contains 1.28M training images and 50K validation images from 1,000 classes. We report single-crop Top-1 accuracy at $224\times224$ resolution. Tables~\ref{tab:imagenet_hierarchical} and~\ref{tab:imagenet_plain} compare hierarchical and plain backbones, respectively. Hierarchical comparisons include the CNN-based backbones ConvNeXt~\citep{Liu_2022_CVPR}, InternImage~\citep{Wang_2023_CVPR}, and MambaOut~\citep{Yu_2025_CVPR}; the ViT-based backbones FasterViT~\citep{ICLR2024_7e496423}, TransNeXt~\citep{Shi_2024_CVPR}, RMT~\citep{Fan_2024_CVPR}, SOFT++~\citep{lu2024softmax}, and MILA~\citep{NEURIPS2024_e618724a}; the SSM-based backbones VMamba~\citep{NEURIPS2024_baa2da9a}, LocalVMamba~\citep{huang2024localmamba}, MambaVision~\citep{Hatamizadeh_2025_CVPR}, and VSSD~\citep{Shi_2025_ICCV}; and the TTT-based H-ViT$^3$~\citep{Han_2026_CVPR}. Plain comparisons include isotropic ConvNeXt~\citep{Liu_2022_CVPR}; the ViT backbones DeiT~\citep{pmlr-v139-touvron21a}, XCiT~\citep{NEURIPS2021_a655fbe4}, and Agent-DeiT~\citep{han2024agent}; the SSM backbones Vim~\citep{pmlr-v235-zhu24f} and Mamba\textsuperscript{\textregistered}~\citep{Wang_2025_CVPR}; and the TTT-based ViT$^3$~\citep{Han_2026_CVPR} and Vision-TTT~\citep{kong2026vision}.
Across both backbone layouts, VisionHOPE achieves the best or tied-best Top-1 accuracy in every scale panel, demonstrating the strong potential of self-modifying learning systems as general-purpose visual backbones.

\begin{table}[t]
\centering
\caption{ImageNet-1K classification of hierarchical backbones at $224\times224$. The panels correspond to increasing model scales. Parameters, FLOPs, and Top-1 accuracy are reported in M, G, and \%.}
\label{tab:imagenet_hierarchical}
\scriptsize
\setlength{\tabcolsep}{2pt}
\renewcommand{\arraystretch}{1.03}
\belowrulesep=0pt
\aboverulesep=0pt
\begin{minipage}[t]{0.327\linewidth}
\centering
\begin{tabular}{lcccc}
\toprule
\textbf{Method} & \textbf{Type} & \textbf{\#P} & \textbf{FLOPs} & \textbf{Acc.} \\
\midrule
ConvNeXt-T & CNN & 29 & 4.5 & 82.1 \\
InternImage-T & CNN & 30 & 5.0 & 83.5 \\
MambaOut-T & CNN & 27 & 4.5 & 82.7 \\
FasterViT-1 & ViT & 53 & 5.3 & 83.2 \\
TransNeXt-T & ViT & 28 & 5.7 & 84.0 \\
RMT-S & ViT & 27 & 4.5 & 84.1 \\
SOFT++-S & ViT & 27 & 4.5 & 82.6 \\
MILA-T & ViT & 25 & 4.2 & 83.5 \\
VMamba-T & SSM & 30 & 4.9 & 82.6 \\
LocalVMamba-T & SSM & 26 & 5.7 & 82.7 \\
MambaVision-T2 & SSM & 35 & 5.1 & 82.7 \\
H-ViT$^3$-T & TTT & 29 & 4.9 & 84.0 \\
\rowcolor{VisionHOPEGreen}
{VisionHOPE-T} & {NL} & {27} & {4.9} & {84.1} \\
\bottomrule
\end{tabular}
\end{minipage}\hfill
\begin{minipage}[t]{0.327\linewidth}
\centering
\begin{tabular}{lcccc}
\toprule
\textbf{Method} & \textbf{Type} & \textbf{\#P} & \textbf{FLOPs} & \textbf{Acc.} \\
\midrule
ConvNeXt-S & CNN & 50 & 8.7 & 83.1 \\
InternImage-S & CNN & 50 & 8.0 & 84.2 \\
MambaOut-S & CNN & 48 & 9.0 & 84.1 \\
FasterViT-2 & ViT & 76 & 8.7 & 84.2 \\
TransNeXt-S & ViT & 50 & 10.3 & 84.7 \\
RMT-B & ViT & 54 & 9.7 & 85.0 \\
SOFT++-M & ViT & 48 & 8.7 & 83.7 \\
MILA-S & ViT & 43 & 7.3 & 84.4 \\
VMamba-S & SSM & 50 & 8.7 & 83.6 \\
LocalVMamba-S & SSM & 50 & 11.4 & 83.7 \\
MambaVision-S & SSM & 50 & 7.5 & 83.3 \\
H-ViT$^3$-S & TTT & 54 & 8.8 & 84.9 \\
\rowcolor{VisionHOPEGreen}
{VisionHOPE-S} & {NL} & {53} & {9.8} & {85.2} \\
\bottomrule
\end{tabular}
\end{minipage}\hfill
\begin{minipage}[t]{0.327\linewidth}
\centering
\begin{tabular}{lcccc}
\toprule
\textbf{Method} & \textbf{Type} & \textbf{\#P} & \textbf{FLOPs} & \textbf{Acc.} \\
\midrule
ConvNeXt-B & CNN & 89 & 15.4 & 83.8 \\
InternImage-B & CNN & 97 & 16.0 & 84.9 \\
MambaOut-B & CNN & 85 & 15.8 & 84.2 \\
FasterViT-3 & ViT & 160 & 18.2 & 84.9 \\
TransNeXt-B & ViT & 90 & 18.4 & 84.8 \\
RMT-L & ViT & 95 & 18.2 & 85.5 \\
SOFT++-L & ViT & 85 & 15.4 & 84.1 \\
MILA-B & ViT & 96 & 16.2 & 85.3 \\
VMamba-B & SSM & 89 & 15.4 & 83.9 \\
MambaVision-B & SSM & 98 & 15.0 & 84.2 \\
VSSD-B & SSM & 89 & 16.1 & 85.4 \\
H-ViT$^3$-B & TTT & 94 & 16.7 & 85.5 \\
\rowcolor{VisionHOPEGreen}
{VisionHOPE-B} & {NL} & {91} & {17.3} & {85.6} \\
\bottomrule
\end{tabular}
\end{minipage}
\end{table}

\begin{table}[t]
\centering
\caption{ImageNet-1K classification of plain backbones at $224\times224$ resolution. The panels correspond to increasing scales. Parameters, FLOPs, and Top-1 accuracy are reported in M, G, and \%.}
\label{tab:imagenet_plain}
\scriptsize
\setlength{\tabcolsep}{1.6pt}
\renewcommand{\arraystretch}{1.03}
\belowrulesep=0pt
\aboverulesep=0pt
\begin{minipage}[t]{0.327\linewidth}
\centering
\begin{tabular}{lcccc}
\toprule
\textbf{Method} & \textbf{Type} & \textbf{\#P} & \textbf{FLOPs} & \textbf{Acc.} \\
\midrule
DeiT-T & ViT & 6 & 1.3 & 72.2 \\
XCiT-T12/16 & ViT & 7 & 1.2 & 77.1 \\
Agent-DeiT-T & ViT & 6 & 1.2 & 74.9 \\
Vim-T & SSM & 7 & 1.6 & 76.1 \\
Mamba\textsuperscript{\textregistered}-T & SSM & 9 & 1.4 & 77.4 \\
ViT$^3$-T & TTT & 6 & 1.2 & 76.5 \\
Vision-TTT-T & TTT & 7 & 1.4 & 77.7 \\
\rowcolor{VisionHOPEGreen}
{P-VisionHOPE-T} & {NL} & {6} & {1.2} & {78.4} \\
\bottomrule
\end{tabular}
\end{minipage}\hfill
\begin{minipage}[t]{0.327\linewidth}
\centering
\begin{tabular}{lcccc}
\toprule
\textbf{Method} & \textbf{Type} & \textbf{\#P} & \textbf{FLOPs} & \textbf{Acc.} \\
\midrule
ConvNeXt-S (iso.) & CNN & 22 & 4.3 & 79.7 \\
DeiT-S & ViT & 22 & 4.6 & 79.8 \\
Agent-DeiT-S & ViT & 23 & 4.4 & 80.5 \\
Vim-S & SSM & 26 & 5.3 & 80.3 \\
Mamba\textsuperscript{\textregistered}-S & SSM & 28 & 5.1 & 81.4 \\
ViT$^3$-S & TTT & 24 & 4.8 & 81.6 \\
Vision-TTT-S & TTT & 26 & 5.3 & 81.8 \\
\rowcolor{VisionHOPEGreen}
{P-VisionHOPE-S} & {NL} & {22} & {4.7} & {82.3} \\
\bottomrule
\end{tabular}
\end{minipage}\hfill
\begin{minipage}[t]{0.327\linewidth}
\centering
\begin{tabular}{lcccc}
\toprule
\textbf{Method} & \textbf{Type} & \textbf{\#P} & \textbf{FLOPs} & \textbf{Acc.} \\
\midrule
ConvNeXt-B (iso.) & CNN & 87 & 16.9 & 82.0 \\
DeiT-B & ViT & 87 & 17.6 & 81.8 \\
Agent-DeiT-B & ViT & 87 & 17.6 & 82.0 \\
Vim-B & SSM & 98 & 19.0 & 81.9 \\
Mamba\textsuperscript{\textregistered}-B & SSM & 99 & 19.8 & 83.0 \\
ViT$^3$-B & TTT & 90 & 18.0 & 82.6 \\
Vision-TTT-B & TTT & 102 & 20.3 & 82.7 \\
\rowcolor{VisionHOPEGreen}
{P-VisionHOPE-B} & {NL} & {88} & {19.0} & {83.4} \\
\bottomrule
\end{tabular}
\end{minipage}
\end{table}

\subsection{Object Detection and Instance Segmentation}
We transfer ImageNet-pretrained hierarchical backbones to COCO using Mask R-CNN~\citep{He_2017_ICCV} and evaluate on val2017. The comparisons include the CNN-based backbones ConvNeXt~\citep{Liu_2022_CVPR}, InternImage~\citep{Wang_2023_CVPR}, and MambaOut~\citep{Yu_2025_CVPR}; the ViT-based backbones Swin~\citep{9710580}, CSWin~\citep{9878609}, SOFT++~\citep{lu2024softmax}, and MILA~\citep{NEURIPS2024_e618724a}; the SSM-based backbones VMamba~\citep{NEURIPS2024_baa2da9a}, LocalVMamba~\citep{huang2024localmamba}, and VSSD~\citep{Shi_2025_ICCV}; and the TTT-based H-ViT$^3$~\citep{Han_2026_CVPR}. Table~\ref{tab:coco_maskrcnn_1x} reports bounding-box and mask AP under the $1\times$ schedule, with FLOPs measured at $1280\times800$ resolution. Detailed results under the $3\times$ schedule are provided in Appendix~\ref{app:coco_3x_results}. VisionHOPE achieves the best or tied-best box and mask AP across all three evaluated model scales, showing that its gains transfer consistently to object detection and instance segmentation.

\begin{table}[t]
\centering
\caption{COCO object detection and instance segmentation with Mask R-CNN under the $1\times$ schedule. The panels correspond to increasing backbone scales. FLOPs and AP are reported in G and \%.}
\label{tab:coco_maskrcnn_1x}
\scriptsize
\setlength{\tabcolsep}{3pt}
\renewcommand{\arraystretch}{1.03}
\belowrulesep=0pt
\aboverulesep=0pt
\begin{minipage}[t]{0.327\linewidth}
\centering
\begin{tabular}{lccc}
\toprule
\textbf{Method} & \textbf{FLOPs} & $\mathbf{AP}^{b}$ & $\mathbf{AP}^{m}$ \\
\midrule
ConvNeXt-T & 262 & 44.2 & 40.1 \\
InternImage-T & 270 & 47.2 & 42.5 \\
MambaOut-T & 262 & 45.1 & 41.0 \\
CSWin-T & 279 & 46.7 & 42.2 \\
SOFT++-S & 261 & 43.8 & 40.1 \\
MILA-T & 255 & 46.8 & 42.1 \\
VMamba-T & 271 & 47.3 & 42.7 \\
VSSD-T & 265 & 46.9 & 42.6 \\
H-ViT$^{3}$-T & 271 & 47.3 & 42.8 \\
\rowcolor{VisionHOPEGreen}
{VisionHOPE-T} & {266} & {47.9} & {43.1} \\
\bottomrule
\end{tabular}
\end{minipage}\hfill
\begin{minipage}[t]{0.327\linewidth}
\centering
\begin{tabular}{lccc}
\toprule
\textbf{Method} & \textbf{FLOPs} & $\mathbf{AP}^{b}$ & $\mathbf{AP}^{m}$ \\
\midrule
ConvNeXt-S & 348 & 45.4 & 41.8 \\
InternImage-S & 340 & 47.8 & 43.3 \\
MambaOut-S & 354 & 47.4 & 42.7 \\
CSWin-S & 342 & 47.9 & 43.2 \\
SOFT++-M & 349 & 46.6 & 42.0 \\
MILA-S & 319 & 49.2 & 44.2 \\
VMamba-S & 349 & 48.7 & 43.7 \\
VSSD-S & 325 & 48.4 & 43.5 \\
H-ViT$^{3}$-S & 349 & 49.1 & 44.1 \\
\rowcolor{VisionHOPEGreen}
{VisionHOPE-S} & {365} & {49.5} & {44.2} \\
\bottomrule
\end{tabular}
\end{minipage}\hfill
\begin{minipage}[t]{0.327\linewidth}
\centering
\begin{tabular}{lccc}
\toprule
\textbf{Method} & \textbf{FLOPs} & $\mathbf{AP}^{b}$ & $\mathbf{AP}^{m}$ \\
\midrule
ConvNeXt-B & 486 & 47.0 & 42.7 \\
InternImage-B & 501 & 48.8 & 44.0 \\
MambaOut-B & 495 & 47.4 & 43.0 \\
Swin-B & 496 & 46.9 & 42.3 \\
CSWin-B & 526 & 48.7 & 43.9 \\
SOFT++-L & 481 & 47.0 & 42.2 \\
MILA-B & 502 & 50.5 & 45.0 \\
VMamba-B & 485 & 49.2 & 44.1 \\
H-ViT$^{3}$-B & 510 & 50.0 & 44.6 \\
\rowcolor{VisionHOPEGreen}
{VisionHOPE-B} & {516} & {50.5} & {45.0} \\
\bottomrule
\end{tabular}
\end{minipage}
\end{table}

\subsection{Semantic Segmentation}
We transfer ImageNet-pretrained hierarchical backbones to ADE20K using UPerNet~\citep{Xiao_2018_ECCV} and evaluate on the validation set. The comparisons include the CNN-based backbones ConvNeXt~\citep{Liu_2022_CVPR} and MambaOut~\citep{Yu_2025_CVPR}; the ViT-based backbones VVT~\citep{10149455} and SOFT++~\citep{lu2024softmax}; the SSM-based MambaVision~\citep{Hatamizadeh_2025_CVPR}; and the TTT-based H-ViT$^3$~\citep{Han_2026_CVPR}. Table~\ref{tab:ade20k_segmentation} reports mIoU, complete-model parameters, and FLOPs, with FLOPs measured at $512\times2048$ resolution. VisionHOPE achieves the best mIoU, providing evidence for the effectiveness of self-modifying learning in semantic segmentation.

\begin{table}[t]
\centering
\caption{ADE20K semantic segmentation with UPerNet. The panels correspond to increasing backbone scales. Complete-model parameters, FLOPs, and mIoU are reported in M, G, and \%.}
\label{tab:ade20k_segmentation}
\scriptsize
\setlength{\tabcolsep}{4pt}
\renewcommand{\arraystretch}{1.03}
\belowrulesep=0pt
\aboverulesep=0pt
\begin{minipage}[t]{0.327\linewidth}
\centering
\begin{tabular}{lccc}
\toprule
\textbf{Backbone} & \textbf{\#P} & \textbf{FLOPs} & \textbf{mIoU} \\
\midrule
ConvNeXt-T & 60 & 939 & 46.0 \\
MambaOut-T & 54 & 938 & 47.4 \\
VVT-S & 56 & 960 & 46.8 \\
SOFT++-S & 60 & 948 & 46.5 \\
MambaVision-T & 55 & 945 & 46.0 \\
H-ViT$^3$-T & 58 & 946 & 48.0 \\
\rowcolor{VisionHOPEGreen}
{VisionHOPE-T} & {55} & {942} & {49.4} \\
\bottomrule
\end{tabular}
\end{minipage}\hfill
\begin{minipage}[t]{0.327\linewidth}
\centering
\begin{tabular}{lccc}
\toprule
\textbf{Backbone} & \textbf{\#P} & \textbf{FLOPs} & \textbf{mIoU} \\
\midrule
ConvNeXt-S & 82 & 1027 & 48.7 \\
MambaOut-S & 76 & 1032 & 49.5 \\
VVT-M & 78 & 1040 & 48.1 \\
SOFT++-M & 81 & 1040 & 48.9 \\
MambaVision-S & 84 & 1135 & 48.2 \\
H-ViT$^3$-S & 84 & 1026 & 50.2 \\
\rowcolor{VisionHOPEGreen}
{VisionHOPE-S} & {82} & {1043} & {50.3} \\
\bottomrule
\end{tabular}
\end{minipage}\hfill
\begin{minipage}[t]{0.327\linewidth}
\centering
\begin{tabular}{lccc}
\toprule
\textbf{Backbone} & \textbf{\#P} & \textbf{FLOPs} & \textbf{mIoU} \\
\midrule
ConvNeXt-B & 122 & 1170 & 49.1 \\
MambaOut-B & 112 & 1178 & 49.6 \\
VVT-L & 92 & 1068 & 48.8 \\
SOFT++-L & 121 & 1204 & 49.2 \\
MambaVision-B & 126 & 1342 & 49.1 \\
H-ViT$^3$-B & 124 & 1195 & 51.7 \\
\rowcolor{VisionHOPEGreen}
{VisionHOPE-B} & {121} & {1200} & {51.8} \\
\bottomrule
\end{tabular}
\end{minipage}
\end{table}

\subsection{Efficiency Analysis}
\paragraph{Computational efficiency.}
For a square feature map containing $N$ tokens of width $D$, we compare the per-layer interaction costs of DeiT~\citep{pmlr-v139-touvron21a}, Vim~\citep{pmlr-v235-zhu24f}, and VisionHOPE, including the associated projections, under the configurations used in our experiments:
\begin{equation}
\begin{aligned}
\mathcal C(\mathrm{DeiT})&=4ND^2+2N^2D,\\
\mathcal C(\mathrm{Vim})&=6ND^2+576ND,\\
\mathcal C(\mathrm{VisionHOPE})&=\frac{5}{4}ND^2+232ND+1600D\sqrt N.
\end{aligned}
\label{eq:visionhope_computational_complexity}
\end{equation}
Self-attention introduces a quadratic interaction term in $N$, whereas Vim and VisionHOPE remain linear in visual token count and scale more favorably as the input image resolution increases.

\paragraph{Memory efficiency.}
For batch size $B$, the corresponding activation-memory complexities are
\begin{equation}
\footnotesize
\mathcal M(\mathrm{DeiT})=O(BND+BN^2),\qquad \mathcal M(\mathrm{Vim})=O(BND),\qquad \mathcal M(\mathrm{VisionHOPE})=O(BND).
\label{eq:visionhope_memory_complexity}
\end{equation}
When materialized explicitly, DeiT's attention matrix introduces a quadratic memory term in $N$, whereas Vim and VisionHOPE retain linear memory scaling as the visual token count increases.

\begin{figure}[t]
\centering
\includegraphics[width=0.99\linewidth]{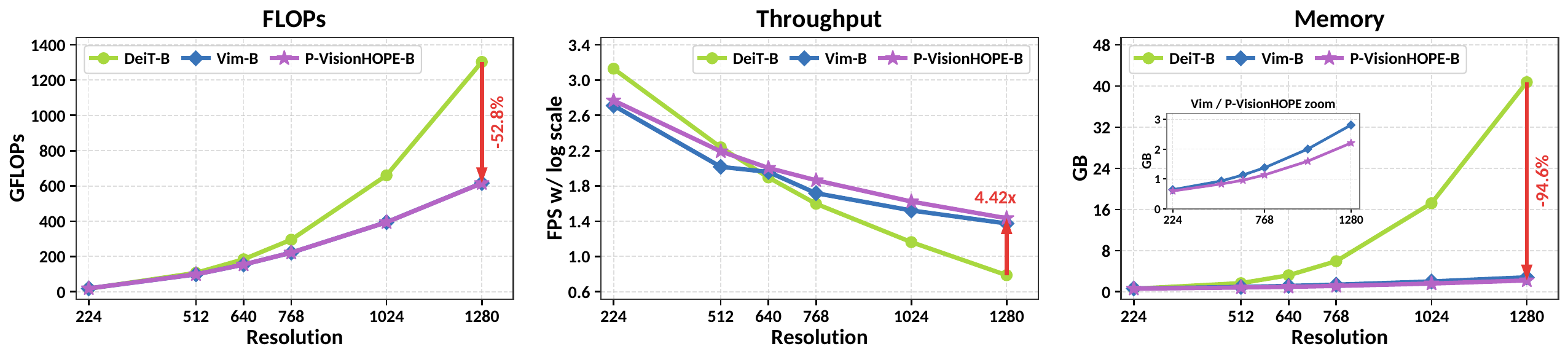}
\caption{Efficiency comparison among DeiT-B, Vim-B, and P-VisionHOPE-B. We plot FLOPs, throughput, and memory footprint for input resolutions ranging from $224\times224$ to $1280\times1280$. Hardware measurements use one NVIDIA A100 GPU with BF16 and a fixed batch size of $8$.}
\label{fig:visionhope_efficiency_base}
\end{figure}

Figure~\ref{fig:visionhope_efficiency_base} empirically confirms these theoretical trends. As input resolution increases, P-VisionHOPE-B maintains linear scaling with visual token count, whereas DeiT-B's computation and memory grow substantially faster. With similar FLOPs to Vim-B, P-VisionHOPE-B achieves higher throughput and lower memory use across all tested resolutions. These results demonstrate the scalability and practical efficiency of VisionHOPE. Appendix~\ref{app:efficiency_details} provides the complete derivations of Equations~(\ref{eq:visionhope_computational_complexity}) and~(\ref{eq:visionhope_memory_complexity}), together with additional efficiency results for the Tiny and Small models.

\subsection{Ablation Study}
We conduct the following ablations on P-VisionHOPE-S while keeping the backbone and training protocol fixed. The first group examines the contributions of the coupled SRNL memories and their DGD update. We use \emph{adaptive K/V} to denote within-image evolution of the key and value memories, and \emph{adaptive dynamics} for that of the learning-rate and retention memories. Removing either leaves the corresponding mappings learned through outer training but fixed within the image, while the content memory continues to evolve. The DGD ablation removes the additional key-dependent correction $-\widetilde\eta_tM^\square_{t-1}k_tk_t^\top$, while retaining the output-space learning signal, retention factor, and stability-matched step-size control. The second group isolates the complementary roles of the soft injection cap and spectral clamp in stabilization. Removing either control leaves the other unchanged. The injection clamp replaces the soft injection cap with a hard clamp at the same injection limit, whereas the soft spectral cap replaces the spectral clamp with a smooth map to the same spectral limit. The final group examines the two-dimensional realization: the one-way variant retains only the forward row-major route, the two-way variant adds its reverse, and the four-way fixed-sum variant uses four routes with unit channel-wise scales. The variant with shared initial states ties the learned initial memories across directions while maintaining separate memory trajectories.

\begin{table}[t]
\centering
\caption{Ablations on P-VisionHOPE-S. The left, middle, and right panels examine the coupled SRNL memories and DGD, stability-matched step-size control, and directional scanning and fusion. Top-1 accuracy is reported in \%. A NaN entry indicates numerical divergence during training.}
\label{tab:visionhope_core_ablations}
\footnotesize
\setlength{\tabcolsep}{4.0pt}
\renewcommand{\arraystretch}{1.03}
\belowrulesep=0pt
\aboverulesep=0pt
\resizebox{\linewidth}{!}{%
\begin{tabular}[t]{lc}
\toprule
\textbf{Variant} & \textbf{Acc.} \\
\midrule
w/o adaptive K/V \& dynamics & 81.7 \\
w/o adaptive K/V & 81.9 \\
w/o adaptive dynamics & 82.0 \\
w/o DGD & 81.9 \\
\rowcolor{VisionHOPEGreen}
{Full model} & 82.3 \\
\bottomrule
\end{tabular}%
\hspace{5pt}%
\begin{tabular}[t]{lc}
\toprule
\textbf{Variant} & \textbf{Acc.} \\
\midrule
w/o soft injection cap & NaN \\
w/o spectral clamp & 82.3 \\
w/ injection clamp & 81.8 \\
w/ soft spectral cap & 81.9 \\
\rowcolor{VisionHOPEGreen}
{Full model} & 82.3 \\
\bottomrule
\end{tabular}%
\hspace{5pt}%
\begin{tabular}[t]{lc}
\toprule
\textbf{Variant} & \textbf{Acc.} \\
\midrule
1-way scan & 81.6 \\
2-way scan & 81.9 \\
4-way fixed sum & 82.1 \\
Shared initial states & 82.0 \\
\rowcolor{VisionHOPEGreen}
{Full model} & 82.3 \\
\bottomrule
\end{tabular}
}
\end{table}

Table~\ref{tab:visionhope_core_ablations} shows consistent benefits from the coupled SRNL memories and their DGD update. Removing the soft injection cap causes the forward pass to diverge numerically early in training, demonstrating its importance for stable optimization. The injection clamp and soft spectral cap remain stable throughout training but are less effective than the complete control scheme. Removing the spectral clamp leaves accuracy unchanged in this setting, which may be related to its extremely low activation rate, as detailed in Appendix~\ref{app:additional_ablations}. Increasing directional coverage improves performance, while channel-wise scales and direction-specific initial states provide further gains. Additional analyses of step-size control behavior and configuration choices are provided in Appendix~\ref{app:additional_ablations}.

\section{Conclusion}
\label{sec:conclusion}
Building on NL, we introduced VisionHOPE, a self-modifying visual backbone whose stored content and learning rule co-evolve through five coupled memories within each image. Our stability-matched step-size control combines a soft injection cap with a spectral clamp to guarantee non-expansive memory dynamics. Four directional scans with spatially aligned chunks enable two-dimensional processing with linear scaling in token count. Experiments on ImageNet-1K, COCO, and ADE20K demonstrate strong performance in classification and dense prediction. These results establish self-modifying learning as a practical foundation for general-purpose visual backbones.

\bibliography{iclr2027_conference}
\bibliographystyle{iclr2027_conference}

\appendix
\section{Related Work}
\label{sec:related_work}
\subsection{Convolutional and Transformer Backbones}
CNNs established the modern visual backbone through local connectivity, shared spatial kernels, and translation-equivariant feature extraction~\citep{NIPS2012_c399862d,He_2016_CVPR}. Modern ConvNets revisit this design space through updated architectural choices in ConvNeXt~\citep{Liu_2022_CVPR} and ConvNeXt V2~\citep{Woo_2023_CVPR}, enlarged receptive fields in RepLKNet~\citep{Ding_2022_CVPR} and UniRepLKNet~\citep{Ding_2024_CVPR}, deformable spatial aggregation in InternImage~\citep{Wang_2023_CVPR}, multi-order gated aggregation in MogaNet~\citep{li2024moganet}, and efficient mobile design in RepViT~\citep{Wang_2024_CVPR}. MambaOut~\citep{Yu_2025_CVPR} further builds gated convolutional backbones by removing the state-space component from Mamba blocks~\citep{gu2023mamba}. These models demonstrate the continued competitiveness of local aggregation. However, the mappings governing that aggregation are learned during training and remain fixed during image processing.

ViT replaces local convolution with content-dependent global self-attention~\citep{NIPS2017_3f5ee243,dosovitskiy2020image}, while Swin Transformer~\citep{9710580} combines hierarchical representations with shifted-window attention. Subsequent backbones broaden this design space by revisiting both backbone organization and the mechanisms used for visual token interaction. MetaFormer~\citep{Yu_2022_CVPR} abstracts the common backbone structure beyond a particular token mixer, and MaxViT~\citep{tu2022maxvit} integrates window and grid attention for local and global interaction. EfficientViT~\citep{cai2022efficientvit} develops multi-scale linear attention for high-resolution dense prediction. RMT~\citep{Fan_2024_CVPR} introduces a retentive mechanism for spatially aware representation, while TransNeXt~\citep{Shi_2024_CVPR} develops foveal token aggregation. More recent work accelerates plain ViTs through a wider global token~\citep{fuller2026thicker}, develops strip self-attention for efficient multi-scale perception~\citep{11298398}, and systematically modernizes the canonical ViT block~\citep{wang2026vit}. Although these models make token interactions dependent on the current image, the learned mechanism generating those interactions does not evolve during inference.

\subsection{Visual State-Space Models}
Structured SSMs provide long-range sequence modeling through recurrent state updates~\citep{gu2021efficiently}. Mamba~\citep{gu2023mamba} introduces input-dependent state transitions and a hardware-aware parallel algorithm, achieving linear sequence-length scaling. Mamba-2~\citep{pmlr-v235-dao24a} connects structured state spaces and attention through state-space duality. Adapting these models to images requires resolving the mismatch between ordered sequences and two-dimensional spatial structure. Vim~\citep{pmlr-v235-zhu24f} uses bidirectional state-space processing, and VMamba~\citep{NEURIPS2024_baa2da9a} introduces four-directional selective scanning over two-dimensional feature maps. Related visual SSMs explore non-hierarchical layouts~\citep{yang2024plainmamba}, multidimensional traversal~\citep{li2024mamba}, windowed or atrous selective scans~\citep{huang2024localmamba,pei2025efficientvmamba}, and non-causal state-space duality~\citep{Shi_2025_ICCV}. More recent work combines redesigned Mamba blocks with attention in a hierarchical backbone~\citep{Hatamizadeh_2025_CVPR}, introduces locally bidirectional recurrence without an additional global reverse scan~\citep{zhang2025lbmamba}, and explores second-order non-causal dynamics that eliminate directional scanning~\citep{ramachandran2026vision}. These methods make the state transition input-dependent, but their transition-generation rule remains fixed: tokens determine transitions through an outer-parameterized mapping unchanged within the image. 

\subsection{Test-Time Training in Vision}
TTT layers use an inner learner as their recurrent state and update it through self-supervised learning~\citep{pmlr-v267-sun25h}. The learner compresses input-derived key-value pairs into its parameters and produces representations from queries. In vision, ViT$^3$~\citep{Han_2026_CVPR} studies full-image inner adaptation, focusing on the learner architecture and inner training procedure. Vision-TTT~\citep{kong2026vision} adapts the learner recurrently along visual token sequences. Its dual-dataset strategy pairs forward and reverse token sequences, while convolutional preprocessing incorporates local spatial information. Together, they establish TTT-based visual backbones in both full-image and token-indexed forms. VisionHOPE shares within-image adaptation with these methods but differs in what evolves. TTT-based visual backbones adapt an inner learner while their representation maps and update dynamics remain largely outer-parameterized. VisionHOPE follows NL's self-referential construction by coupling content, key, value, learning-rate, and retention memories~\citep{NEURIPS2025_4309616a}. This coupling allows stored content and the learning rule to co-evolve with visual context.

\subsection{Nested Learning and Self-Modifying Systems}
Fast-weight models update short-term parameters from the current context to capture recent history~\citep{NIPS2016_9f44e956}, while HyperNetworks generate parameters of one network using another~\citep{ha2017hypernetworks}. Learned optimizers make update rules trainable~\citep{NIPS2016_fb875828}, and fast-weight programmers connect linear attention and recurrent sequence modeling with online associative memory~\citep{pmlr-v139-schlag21a,NEURIPS2021_3f9e3767}. Self-referential weight matrices generate their update quantities and modify themselves through the delta rule~\citep{irie2022modern}. 
Titans~\citep{NEURIPS2025_a4ca07aa} treats neural memory as a learner updated at test time to retain information over long contexts. These works establish key ingredients of adaptive computation: context-dependent parameters, trainable update rules, and memory implemented as a learning process during inference.

Building on these ideas, NL describes a model as a collection of interconnected learning processes, each compressing its context flow into an internal state~\citep{NEURIPS2025_4309616a}. Its self-referential construction couples five memories that store content, generate key and value representations, and govern learning rate and retention. This coupling allows stored content and the learning rule to co-evolve as context accumulates. NL also derives a chunk formulation and combines its self-referential component with a Continuum Memory System (CMS) to form the HOPE architecture. CMS updates its memory modules at different frequencies to retain information over multiple timescales. In NL's formulation, the learning signal for these modules comes from a task objective. Language modeling provides this signal through next-token prediction, whereas discriminative vision has no canonical token-wise target along a visual scan. Incorporating CMS into a visual backbone would therefore require choosing an additional online learning objective. VisionHOPE focuses on the self-referential component and adapts it to generic visual backbones. In vision, the unconstrained self-referential update creates a stability barrier, which VisionHOPE addresses with a step-size control scheme.

\section{Additional Methodological Details}
\label{app:method_details}

\subsection{Gradient Descent as a Dot-Product Associative Memory}
\label{app:gd_associative_memory}
For completeness, we expand the dot-product associative-memory formulation of gradient descent introduced by NL. Here, $t$ indexes the current element of the context flow, corresponding to a visual token in VisionHOPE, and $d$ is the shared key and value dimension. At step $t$, $W_{t-1}\in\R^{d\times d}$ is the existing memory state and $k_t\in\R^d$ is the current key. The output-space learning signal $g_t\in\R^d$ is the gradient of the local objective at the response $W_{t-1}k_t$, and $\eta_t>0$ is the learning rate. With $k_t$, $g_t$, and $\eta_t$ held fixed, the update in Equation~(\ref{eq:nl_gd_update}) is the solution to the proximal problem
\begin{equation}
W_t^{\mathrm{GD}}=\arg\min_{W\in\R^{d\times d}}\left[\left\langle Wk_t,g_t\right\rangle+\frac{1}{2\eta_t}\left\lVert W-W_{t-1}\right\rVert_F^2\right],
\label{eq:app_gd_dot_product_objective}
\end{equation}
where $W\in\R^{d\times d}$ is a candidate updated memory, and $\langle\cdot,\cdot\rangle$ is the Euclidean inner product. The dot-product term encourages the response $Wk_t$ to move along the write signal $-g_t$, while the quadratic penalty discourages large changes from $W_{t-1}$. Differentiating the objective with respect to $W$ gives
\begin{equation}
\nabla_W\left\langle Wk_t,g_t\right\rangle=g_tk_t^\top,\qquad \nabla_W\frac{1}{2\eta_t}\left\lVert W-W_{t-1}\right\rVert_F^2=\eta_t^{-1}(W-W_{t-1}).
\label{eq:app_gd_objective_gradients}
\end{equation}
Setting their sum to zero at the minimizer $W_t^{\mathrm{GD}}$ yields
\begin{equation}
g_tk_t^\top+\eta_t^{-1}(W_t^{\mathrm{GD}}-W_{t-1})=0\quad\Longrightarrow\quad W_t^{\mathrm{GD}}=W_{t-1}-\eta_tg_tk_t^\top,
\label{eq:app_gd_closed_form}
\end{equation}
which recovers Equation~(\ref{eq:nl_gd_update}). Since $\eta_t>0$, the objective is strictly convex and the minimizer is unique. Thus, $k_t$ specifies the write address and $-g_t$ provides the write signal. Replacing the dot-product term with an $L_2$ regression objective yields the DGD formulation derived in Appendix~\ref{app:dgd_proximal_derivation}.

\subsection{Proximal Derivation of Delta Gradient Descent}
\label{app:dgd_proximal_derivation}
We derive the recurrence in Equation~(\ref{eq:nl_dgd_update}) from the proximal $L_2$ objective in Equation~(\ref{eq:nl_dgd_proximal_objective}). Here, $u_t=-g_t$ is the regression target and $\lambda_t>0$ is the proximal parameter. With $k_t$, $u_t$, and $\lambda_t$ held fixed, differentiating the two terms of the objective with respect to the candidate memory $W$ gives
\begin{equation}
\nabla_W\frac{1}{2}\left\lVert Wk_t-u_t\right\rVert_2^2=\left(Wk_t-u_t\right)k_t^\top,\qquad \nabla_W\frac{1}{2\lambda_t}\left\lVert W-W_{t-1}\right\rVert_F^2=\lambda_t^{-1}\left(W-W_{t-1}\right).
\label{eq:app_dgd_objective_gradients}
\end{equation}
Since $\lambda_t>0$, the objective is strictly convex with respect to $W$. Setting the sum of these gradients to zero therefore gives the first-order condition for its unique minimizer $W_t^{\mathrm{DGD}}$:
\begin{equation}
\left(W_t^{\mathrm{DGD}}k_t-u_t\right)k_t^\top+\lambda_t^{-1}\left(W_t^{\mathrm{DGD}}-W_{t-1}\right)=0.
\label{eq:app_dgd_first_order}
\end{equation}
Multiplying by $\lambda_t$ and collecting the terms containing $W_t^{\mathrm{DGD}}$ gives
\begin{equation}
W_t^{\mathrm{DGD}}\left(I_d+\lambda_tk_tk_t^\top\right)=W_{t-1}+\lambda_tu_tk_t^\top.
\label{eq:app_dgd_linear_system}
\end{equation}
Here, $I_d\in\R^{d\times d}$ is the identity matrix. The matrix $I_d+\lambda_tk_tk_t^\top$ is positive definite and therefore invertible. Applying the Sherman-Morrison identity gives
\begin{equation}
\left(I_d+\lambda_tk_tk_t^\top\right)^{-1}=I_d-\frac{\lambda_t}{1+\lambda_t\lVert k_t\rVert_2^2}k_tk_t^\top=I_d-\eta_tk_tk_t^\top,
\label{eq:app_dgd_sherman_morrison}
\end{equation}
where $\eta_t=\lambda_t/(1+\lambda_t\lVert k_t\rVert_2^2)$ is the effective DGD step in Equation~(\ref{eq:nl_dgd_effective_step}). NL simplifies this coefficient for normalized keys, while we keep the key norm explicit to cover the unnormalized case. Right-multiplying both sides of Equation~(\ref{eq:app_dgd_linear_system}) by this inverse and using $\lambda_t(1-\eta_t\lVert k_t\rVert_2^2)=\eta_t$ gives
\begin{equation}
W_t^{\mathrm{DGD}}=W_{t-1}\left(I_d-\eta_tk_tk_t^\top\right)+\eta_tu_tk_t^\top=W_{t-1}-\eta_t\left(W_{t-1}k_t-u_t\right)k_t^\top.
\label{eq:app_dgd_proximal_closed_form}
\end{equation}
Substituting $u_t=-g_t$ into this expression recovers Equation~(\ref{eq:nl_dgd_update}).

\subsection{Properties of the Stability-Matched Step-Size Control}
\label{app:projection_properties}
This subsection establishes the token-wise guarantees of the stability-matched step-size control. We first characterize the soft injection cap, then prove complementary operator bounds for the retained memory transition and self-referential injection, and finally combine them to obtain token-wise non-expansion. We conclude with the finite-precision safeguards used in the practical implementation.

Since $\eta_t>0$ and $\eta_t^{\mathrm{inj}}>0$, their ratio is positive, giving $0<1-\exp(-\eta_t/\eta_t^{\mathrm{inj}})<1$ and hence $0<\bar\eta_t<\eta_t^{\mathrm{inj}}$. Differentiating with respect to the raw step while holding $\eta_t^{\mathrm{inj}}$ fixed gives
\begin{equation}
\frac{\partial\bar\eta_t}{\partial\eta_t}=\exp\left(-\frac{\eta_t}{\eta_t^{\mathrm{inj}}}\right)\in(0,1).
\label{eq:app_soft_cap_derivative}
\end{equation}
Thus, the soft injection cap is smooth and strictly increasing in $\eta_t$. Expanding around $\eta_t=0$ yields
\begin{equation}
\bar\eta_t=\eta_t-\frac{\eta_t^2}{2\eta_t^{\mathrm{inj}}}+O\left(\frac{\eta_t^3}{(\eta_t^{\mathrm{inj}})^2}\right),
\label{eq:app_soft_cap_expansion}
\end{equation}
which proves first-order preservation of small raw steps. The spectral clamp in Equation~(\ref{eq:visionhope_spectral_clamp}) is piecewise smooth and differentiable almost everywhere with respect to its two scalar inputs.

\begin{proof}[Proof of Proposition~\ref{prop:visionhope_operator_bounds}]
If $k_t=0$, then $A_t=\alpha_tI_d$ and $B_t=0$, so both required operator bounds hold directly. It therefore remains to consider $k_t\neq0$. With $\kappa_t=\lVert k_t\rVert_2^2$, the spectral clamp ensures
\begin{equation}
0\leq\widetilde\eta_t\kappa_t\leq\eta_t^{\mathrm{spec}}\kappa_t=2\alpha_t.
\label{eq:app_spectral_step_bound}
\end{equation}
The rank-one matrix $k_tk_t^\top$ has eigenvalue $\kappa_t$ along $k_t$ and zero on its orthogonal complement. Consequently, the retained transition $A_t=\alpha_tI_d-\widetilde\eta_tk_tk_t^\top$ has eigenvalue $\alpha_t-\widetilde\eta_t\kappa_t$ along $k_t$ and $\alpha_t$ in every orthogonal direction. Its spectral bound is therefore equivalent to the following condition:
\begin{equation}
\lVert A_t\rVert_2\leq\alpha_t\quad\Longleftrightarrow\quad\left|\alpha_t-\widetilde\eta_t\kappa_t\right|\leq\alpha_t\quad\Longleftrightarrow\quad0\leq\widetilde\eta_t\kappa_t\leq2\alpha_t.
\label{eq:app_retained_transition_equivalence}
\end{equation}
Together with Equation~(\ref{eq:app_spectral_step_bound}), this proves $\lVert A_t\rVert_2\leq\alpha_t$. For the self-referential injection, the rank-one matrix $\delta_tk_t^\top$ has spectral norm $\lVert\delta_t\rVert_2\lVert k_t\rVert_2$, giving
\begin{equation}
\lVert B_t\rVert_2=\widetilde\eta_t\lVert\delta_t\rVert_2\lVert k_t\rVert_2\leq\bar\eta_t\lVert\delta_t\rVert_2\lVert k_t\rVert_2\leq\eta_t^{\mathrm{inj}}\lVert\delta_t\rVert_2\lVert k_t\rVert_2\leq(1-\alpha_t)\frac{\lVert\delta_t\rVert_2}{r_t}<1-\alpha_t.
\label{eq:app_injection_bound_proof}
\end{equation}
Here, the first two inequalities follow from $\widetilde\eta_t\leq\bar\eta_t<\eta_t^{\mathrm{inj}}$, the third uses $\lVert k_t\rVert_2\leq1$ and $\eta_t^{\mathrm{inj}}=(1-\alpha_t)/r_t$, and the final inequality is strict because $r_t>\lVert\delta_t\rVert_2$. The argument also covers $\delta_t=0$, for which $B_t=0$. This proves the strict bound $\lVert B_t\rVert_2<1-\alpha_t$ at every token.
\end{proof}

\begin{proof}[Proof of Corollary~\ref{cor:visionhope_token_nonexpansion}]
By Proposition~\ref{prop:visionhope_operator_bounds} and the triangle inequality,
\begin{equation}
\lVert T_t\rVert_2\leq\lVert A_t\rVert_2+\lVert B_t\rVert_2<\alpha_t+(1-\alpha_t)=1.
\label{eq:app_complete_transition_bound}
\end{equation}
Since $M_t^\square=M_{t-1}^\square T_t$, submultiplicativity gives
\begin{equation}
\lVert M_t^\square\rVert_F\leq\lVert M_{t-1}^\square\rVert_F\lVert T_t\rVert_2\leq\lVert M_{t-1}^\square\rVert_F.
\label{eq:app_token_memory_nonexpansion}
\end{equation}
Applying this argument at every token proves non-expansion along the fully token-indexed trajectory, so each memory-state norm remains bounded by its initial value. The result applies to all five memories, with the learning-rate and retention memories treated as $1\times d$ linear maps.
\end{proof}

In finite precision, VisionHOPE applies fixed multiplicative inward factors of $1-10^{-3}$ and $1-10^{-6}$ to the injection and spectral limits, respectively. The squared key norm in the denominator of $\eta_t^{\mathrm{spec}}$ is explicitly lower-bounded by $10^{-20}$, and the resulting spectral limit is rounded to the next representable value toward zero. These safeguards make the implemented limits slightly more conservative and provide additional margins against roundoff errors near both step-size limits.

\subsection{Chunk-Wise Stability of VisionHOPE}
\label{app:chunk_stability}
We apply the step-size control to the chunk-wise recurrence in Section~\ref{sec:nl_chunk_formulation}. Fix a chunk spanning tokens $b+1,\ldots,b+C$, where $b$ is a nonnegative multiple of $C$. For these tokens, $C\times\left\lfloor\frac{t-1}{C}\right\rfloor=b$, so the boundary state of memory $\square\in\{m,k,v,\eta,\alpha\}$ is $M_b^\square$. Let $\delta_t=k_t-v_t$ and obtain $\widetilde\eta_t$ by applying Equations~(\ref{eq:visionhope_soft_projection}) and~(\ref{eq:visionhope_spectral_clamp}) to the quantities generated from these boundary states. We use $A_t$ and $B_t$ defined in Equation~(\ref{eq:visionhope_projected_transition}), whose bounds in Proposition~\ref{prop:visionhope_operator_bounds} remain valid for these quantities. Because $g_t^\square=M_b^\square\delta_t$, substituting these quantities and the executed step $\widetilde\eta_t$ into Equation~(\ref{eq:nl_online_srdgd}) gives
\begin{equation}
M_t^\square=M_{t-1}^\square A_t+M_b^\square B_t,\qquad t=b+1,\ldots,b+C.
\label{eq:app_projected_chunk_recurrence}
\end{equation}

\begin{proposition}[Shared right-multiplicative closure]
\label{prop:app_shared_gain}
For this fixed chunk, every memory $\square\in\{m,k,v,\eta,\alpha\}$ admits the following shared right-multiplicative representation:
\begin{equation}
M_t^\square=M_b^\square G_t,\qquad G_b=I_d,\qquad G_t=G_{t-1}A_t+B_t,
\label{eq:app_shared_gain_recurrence}
\end{equation}
where the recurrence holds for $t=b+1,\ldots,b+C$ and the same gain $G_t\in\R^{d\times d}$ is shared by all five memories. The gain is defined relative to the boundary of this fixed chunk.
\end{proposition}

\begin{proof}
The claim holds at $t=b$ because $M_b^\square=M_b^\square I_d$. For the inductive step, suppose $M_{t-1}^\square=M_b^\square G_{t-1}$. Substituting this relation into Equation~(\ref{eq:app_projected_chunk_recurrence}) gives
\begin{equation}
M_t^\square=M_b^\square\left(G_{t-1}A_t+B_t\right)=M_b^\square G_t.
\label{eq:app_shared_gain_induction}
\end{equation}
The matrices $A_t$ and $B_t$ do not depend on $\square$, so the resulting gain is shared by all memories.
\end{proof}

\begin{theorem}[Non-expansion of the chunk-wise recurrence]
\label{thm:app_chunk_nonexpansion}
Under the stability-matched step-size control, every within-chunk gain and memory state satisfies
\begin{equation}
\lVert G_t\rVert_2\leq1,\qquad \lVert M_t^\square\rVert_F\leq\lVert M_b^\square\rVert_F,\qquad t=b,\ldots,b+C.
\label{eq:app_chunk_nonexpansion}
\end{equation}
\end{theorem}

\begin{proof}
The case $t=b$ follows from $G_b=I_d$. For $b<t\leq b+C$, Proposition~\ref{prop:app_shared_gain} reduces all five memory trajectories to the shared-gain recurrence in Equation~(\ref{eq:app_shared_gain_recurrence}). Unrolling this recurrence gives
\begin{equation}
G_t=A_{b+1}\cdots A_t+\sum_{j=b+1}^{t}B_jA_{j+1}\cdots A_t,
\label{eq:app_chunk_gain_unrolling}
\end{equation}
where empty matrix and scalar products are understood as $I_d$ and $1$, respectively, and the multiplication order is inherited from the recurrence. Applying submultiplicativity and Proposition~\ref{prop:visionhope_operator_bounds} yields
\begin{equation}
\lVert G_t\rVert_2<\prod_{s=b+1}^{t}\alpha_s+\sum_{j=b+1}^{t}(1-\alpha_j)\prod_{s=j+1}^{t}\alpha_s.
\label{eq:app_chunk_gain_bound}
\end{equation}
The inequality is strict because the final injection term, corresponding to $j=t$, retains the strict bound $\lVert B_t\rVert_2<1-\alpha_t$. The accumulated injection terms telescope because
\begin{equation}
(1-\alpha_j)\prod_{s=j+1}^{t}\alpha_s=\prod_{s=j+1}^{t}\alpha_s-\prod_{s=j}^{t}\alpha_s.
\label{eq:app_chunk_telescoping_term}
\end{equation}
Therefore,
\begin{equation}
\prod_{s=b+1}^{t}\alpha_s+\sum_{j=b+1}^{t}(1-\alpha_j)\prod_{s=j+1}^{t}\alpha_s=1,
\label{eq:app_chunk_telescoping_identity}
\end{equation}
and Equation~(\ref{eq:app_chunk_gain_bound}) gives $\lVert G_t\rVert_2<1$ for every $b<t\leq b+C$. Finally, $M_t^\square=M_b^\square G_t$ and Frobenius-spectral submultiplicativity give $\lVert M_t^\square\rVert_F\leq\lVert M_b^\square\rVert_F\lVert G_t\rVert_2\leq\lVert M_b^\square\rVert_F$. In particular, the state at the end of the chunk satisfies $\lVert M_{b+C}^\square\rVert_F\leq\lVert M_b^\square\rVert_F$. This state becomes the boundary state for the next chunk, where the same argument applies with a new gain initialized to $I_d$. Repeating the argument bounds the norm of every memory state along the scan by that of its initial state.
\end{proof}

\subsection{VisionHOPE Architecture Details}
\label{app:visionhope_architecture_details}
Section~\ref{sec:visionhope_architecture} presents the structural form of the VisionHOPE operator and block. In the reported models, the block input first undergoes Conditional Positional Encoding (CPE)~\citep{chu2023conditional} and pre-normalization before a pointwise input projection maps it to the internal width. The $\operatorname{DW}$ operation in Equation~(\ref{eq:visionhope_simplified_block}) is implemented as a gated depthwise-convolutional residual branch. It supplies local visual context and produces the feature map $Z$ used by the VisionHOPE operator.

Along each route, we divide the $D_m$ channels of $Z_r$ into $n_h=D_m/d$ heads of dimension $d$. Each head maintains an independent five-memory SRNL state for each direction. In the multi-head implementation, we apply a $D_m\times D_m$ outer query projection once to $Z$ and reuse the resulting queries across all four scan directions. For each direction, we reorder the queries into scan order, split them into heads, and apply $L_2$ normalization within each head before content readout.

After the recurrent scans, the head-wise content reads are concatenated and combined with a learned channel-wise skip connection from $Z_r$ to form $Y_r$. The sequences are restored to the feature grid and fused using the direction scales in Equation~(\ref{eq:visionhope_directional_fusion}), yielding $Y$. A depthwise local positional branch processes $Z$ and adds its output to $Y$, forming a skip connection around the VisionHOPE operator. A pointwise output projection then maps the combined output back to the original width $D$.

At the block level, both residual branches use LayerNorm, with stochastic depth applied during training. The hierarchical backbone arranges VisionHOPE blocks across progressively downsampled stages, whereas the plain backbone stacks them at a single resolution. Row- and column-aligned chunk lengths follow each stage's feature resolution. Appendix~\ref{app:model_configurations} provides model configurations.

\section{Additional Experimental Details}
\label{app:additional_experiments}

\subsection{Model Configurations}
\label{app:model_configurations}

\paragraph{Hierarchical models.}
VisionHOPE-T/S/B distribute VisionHOPE blocks across four progressively downsampled stages. Model scale is determined by the stage widths and depths. The FFN expansion ratio is set to $4$ throughout. Table~\ref{tab:visionhope_hierarchical_configs} specifies the complete hierarchical configurations.

\begin{table}[t]
\centering
\caption{Configurations of hierarchical VisionHOPE models. For an input of size $H\times W$, $\downarrow n$ denotes downsampling by a factor of $n$. $\mathrm{B}(D,D_m,n_h)\!\times\!L$ denotes $L$ repeated VisionHOPE blocks with block width $D$, internal width $D_m$, and $n_h$ heads per direction. The head dimension is $d=16$.}
\label{tab:visionhope_hierarchical_configs}
\footnotesize
\setlength{\tabcolsep}{9.3pt}
\renewcommand{\arraystretch}{1.05}
\belowrulesep=0pt
\aboverulesep=0pt
\begin{tabular}{c|c|c|c|c}
\toprule
 & \textbf{Size} & \textbf{VisionHOPE-T} & \textbf{VisionHOPE-S} & \textbf{VisionHOPE-B} \\
\midrule
\multirow[c]{2}{*}{Stage 1} & \multirow[c]{2}{*}{$\left(\frac{H}{4}\times \frac{W}{4}\right)$} & Stem $\downarrow\!4$ & Stem $\downarrow\!4$ & Stem $\downarrow\!4$ \\
& & $\mathrm{B}(64,32,2)\!\times\!3$ & $\mathrm{B}(64,32,2)\!\times\!4$ & $\mathrm{B}(96,48,3)\!\times\!4$ \\
\midrule
\multirow[c]{2}{*}{Stage 2} & \multirow[c]{2}{*}{$\left(\frac{H}{8}\times \frac{W}{8}\right)$} & Down $\downarrow\!2$ & Down $\downarrow\!2$ & Down $\downarrow\!2$ \\
& & $\mathrm{B}(128,64,4)\!\times\!4$ & $\mathrm{B}(160,80,5)\!\times\!8$ & $\mathrm{B}(192,96,6)\!\times\!8$ \\
\midrule
\multirow[c]{2}{*}{Stage 3} & \multirow[c]{2}{*}{$\left(\frac{H}{16}\times \frac{W}{16}\right)$} & Down $\downarrow\!2$ & Down $\downarrow\!2$ & Down $\downarrow\!2$ \\
& & $\mathrm{B}(256,128,8)\!\times\!18$ & $\mathrm{B}(320,160,10)\!\times\!25$ & $\mathrm{B}(448,224,14)\!\times\!25$ \\
\midrule
\multirow[c]{2}{*}{Stage 4} & \multirow[c]{2}{*}{$\left(\frac{H}{32}\times \frac{W}{32}\right)$} & Down $\downarrow\!2$ & Down $\downarrow\!2$ & Down $\downarrow\!2$ \\
& & $\mathrm{B}(512,256,16)\!\times\!4$ & $\mathrm{B}(512,256,16)\!\times\!8$ & $\mathrm{B}(640,320,20)\!\times\!8$ \\
\midrule
Classifier & -- & \multicolumn{3}{c}{Global average pooling, linear} \\
\bottomrule
\end{tabular}
\end{table}

\paragraph{Plain models.}
P-VisionHOPE-T/S/B retain a single feature resolution throughout the backbone after the convolutional stem. We scale these models by increasing the block width and total depth. The FFN expansion ratio is set to $4$ throughout. Table~\ref{tab:visionhope_plain_configs} specifies the complete plain architectures.

\begin{table}[t]
\centering
\caption{Configurations of plain VisionHOPE models. For an input of size $H\times W$, Stem $\downarrow16$ produces a feature map at $\left(\frac{H}{16}\times \frac{W}{16}\right)$. $\mathrm{B}(D,D_m,n_h)\!\times\!L$ denotes $L$ repeated VisionHOPE blocks with block width $D$, internal width $D_m$, and $n_h$ heads per direction. The head dimension is $d=16$.}
\label{tab:visionhope_plain_configs}
\footnotesize
\setlength{\tabcolsep}{9.5pt}
\renewcommand{\arraystretch}{1.05}
\belowrulesep=0pt
\aboverulesep=0pt
\begin{tabular}{c|c|c|c|c}
\toprule
 & \textbf{Size} & \textbf{P-VisionHOPE-T} & \textbf{P-VisionHOPE-S} & \textbf{P-VisionHOPE-B} \\
\midrule
\multirow[c]{2}{*}{Backbone} & \multirow[c]{2}{*}{$\left(\frac{H}{16}\times \frac{W}{16}\right)$} & Stem $\downarrow\!16$ & Stem $\downarrow\!16$ & Stem $\downarrow\!16$ \\
& & $\mathrm{B}(192,96,6)\!\times\!12$ & $\mathrm{B}(384,192,12)\!\times\!13$ & $\mathrm{B}(768,384,24)\!\times\!14$ \\
\midrule
Classifier & -- & \multicolumn{3}{c}{Global average pooling, linear} \\
\bottomrule
\end{tabular}
\end{table}

\paragraph{Shared VisionHOPE settings.}
At $224\times224$ resolution, row and column chunk lengths coincide: $(56,28,14,7)$ for the four hierarchical stages and $14$ for the plain models. At other resolutions, row chunks span the feature map's width and column chunks its height. For the scalar outputs $s_t^\eta$ and $s_t^\alpha$ of the learning-rate and retention memories, the maps in Equations~(\ref{eq:nl_self_referential_quantities}) and~(\ref{eq:nl_chunk_generated_quantities}) are given by
\begin{equation}
\eta_t=\phi_\eta(s_t^\eta)=\gamma_\eta\,\operatorname{softplus}(s_t^\eta),\qquad \alpha_t=\phi_\alpha(s_t^\alpha)=\sigma\!\left(s_t^\alpha+\log\!\frac{\alpha_{\mathrm{init}}}{1-\alpha_{\mathrm{init}}}\right).
\label{eq:app_visionhope_scalar_maps}
\end{equation}
Here, $\sigma$ denotes the logistic sigmoid, $\gamma_\eta>0$ is the raw-step scale, and $\alpha_{\mathrm{init}}\in(0,1)$ is the initial retention value. We set $\gamma_\eta=0.025$ and $\alpha_{\mathrm{init}}=0.9$. The stability analysis requires no additional property of these scalar maps beyond $\eta_t>0$ and $\alpha_t\in[0,1)$. Their particular parameterization therefore does not alter the complementary operator bounds or the resulting non-expansion guarantees for the coupled memory dynamics. Before outer training, $M^m_0$, $M^k_0$, and $M^v_0$ are initialized from a truncated normal distribution with standard deviation $0.01$, whereas $m^\eta_0$ and $m^\alpha_0$ are initialized to zero. Zero-initializing the scalar memories does not make the generated step zero: the initial raw step and retention factor are $\gamma_\eta\log 2$ and $\alpha_{\mathrm{init}}$, respectively. Setting $\alpha_{\mathrm{init}}=0.9$ preserves most of the existing memory while leaving the initial stability margin $1-\alpha_{\mathrm{init}}=0.1$ for self-referential injection. Stabilized key normalization uses $k_t=k_t^{\mathrm{raw}}/\sqrt{\lVert k_t^{\mathrm{raw}}\rVert_2^2+10^{-6}}$, where $k_t^{\mathrm{raw}}$ denotes the key memory output before normalization. The step-size control in Definition~\ref{def:stability_matched_projection} uses $\epsilon=10^{-6}$.

\subsection{Training Protocols}
\label{app:implementation_details}

\paragraph{Image classification.}
We train our ImageNet models with AdamW~\citep{loshchilov2019decoupled}. The $300$-epoch schedule begins with $20$ epochs of linear learning-rate warmup, followed by cosine decay. The learning rate is scaled linearly from $10^{-3}$ at a global batch size of $1024$. The initial warmup and minimum learning rates are $10^{-6}$ and $10^{-5}$, respectively. The hierarchical models undergo $10$ additional training epochs at the minimum learning rate. The stochastic-depth rates are $(0.2,0.4,0.4)$ for VisionHOPE-T/S/B and $(0,0.1,0.5)$ for P-VisionHOPE-T/S/B. The hierarchical models additionally use MESA training~\citep{NEURIPS2022_948b1c9d}. Beginning at epoch $75$, we use an exponential moving average teacher with decay $0.9998$ to provide soft class targets for the student model. The corresponding soft cross-entropy loss is added to the classification objective with weights $1.0$, $1.5$, and $1.5$ for the Tiny, Small, and Base models. The plain models are trained without MESA.

\paragraph{Object detection and instance segmentation.}
We initialize each Mask R-CNN model from its corresponding ImageNet-pretrained hierarchical backbone and fine-tune it with AdamW and standard multi-scale augmentation. The base learning rate is $4\times10^{-4}$ for a global batch size of $64$ and is scaled linearly when the batch size changes. The $1\times$ and $3\times$ schedules run for $12$ and $36$ epochs, respectively, with linear warmup over the first $1000$ iterations and tenfold learning-rate reductions at epochs $(8,11)$ and $(27,33)$. The stochastic-depth rates are $(0.2,0.3,0.3)$ for VisionHOPE-T/S/B under the $1\times$ schedule and $(0.2,0.4)$ for the Tiny and Small models under the $3\times$ schedule.

\paragraph{Semantic segmentation.}
We initialize UPerNet from the corresponding ImageNet-pretrained hierarchical backbones and optimize each model for $160$K iterations with AdamW and a base learning rate of $6\times10^{-5}$. Training uses a global batch size of $16$, linear warmup over the first $1500$ iterations, and polynomial learning-rate decay with power $1.0$. We follow the standard ADE20K data augmentation and use stochastic-depth rates of $(0.2,0.4,0.4)$ for VisionHOPE-T/S/B, respectively.

\subsection{Efficiency Analysis Details}
\label{app:efficiency_details}
\paragraph{Computational complexity.}
Equation~(\ref{eq:visionhope_computational_complexity}) directly compares the theoretical per-layer interaction costs for $N$ tokens of width $D$. These estimates include the main linear projections and attention or recurrent computations, isolating token interaction from FFNs, local convolutions, and auxiliary block operations. 
In DeiT, the $4ND^2$ term comes from the standard query, key, value, and output projections, while $2N^2D$ accounts for constructing the dense attention matrix and applying it to the values.
For Vim, let $E$ denote the expanded inner width, $S$ the hidden state dimension, and $R$ the number of recurrent directions. The input and output projections contribute $2NDE$ and $NDE$, respectively, while its recurrent selective state-space computation contributes $9RNES$. Thus,
\begin{equation}
\mathcal C(\mathrm{Vim})=3NDE+9RNES.
\label{eq:app_vim_general_interaction_cost}
\end{equation}
The Vim configurations considered here use $E=2D$, $S=16$, and $R=2$, yielding
\begin{equation}
\mathcal C(\mathrm{Vim})=6ND^2+576ND.
\label{eq:app_vim_interaction_cost}
\end{equation}
For VisionHOPE, let $D_m$ denote the internal width, $d$ the memory head dimension, $n_h=D_m/d$ the number of heads per direction, and $C_r$ the chunk length along route $r\in\mathcal D$. The input and output projections, the outer query projection, and the SRNL instances applied along the four routes give
\begin{equation}
\mathcal C(\mathrm{VisionHOPE})=2NDD_m+ND_m^2+4Nn_h(6d^2+20d)+n_hN(3d^3+2d^2)\sum_{r\in\mathcal D}\frac{1}{C_r}.
\label{eq:app_visionhope_general_interaction_cost}
\end{equation}
The first term counts the input and output projections, and the second counts the shared outer query projection applied once to the internal feature map. 
The third covers the generation of token-dependent quantities, the shared-gain recurrence with the executed step, and content reads across the four routes, while the final term accounts for updating the five memory states at chunk boundaries. The evaluated models use $D_m=D/2$, $d=16$, and $n_h=D/32$. For an $H\times W$ map with $N=HW$, the row-major routes use $C_r=W$ and the column-major routes use $C_r=H$, making the boundary term $800D(H+W)$. For square maps with $H=W=\sqrt N$, Equation~(\ref{eq:app_visionhope_general_interaction_cost}) becomes
\begin{equation}
\mathcal C(\mathrm{VisionHOPE})=\frac{5}{4}ND^2+232ND+1600D\sqrt N.
\label{eq:app_visionhope_interaction_cost}
\end{equation}

\paragraph{Memory complexity.}
Equation~(\ref{eq:visionhope_memory_complexity}) compares the theoretical per-layer activation-memory requirements for batch size $B$. DeiT stores $O(BND)$ intermediate token representations and, when materialized explicitly, an $O(BN^2)$ attention matrix. Using the same $E$, $S$, and $R$ defined above, the token and direction-wise state-space activations in Vim give the following memory bound
\begin{equation}
\mathcal M(\mathrm{Vim})=O(BND+BRNES).
\label{eq:app_vim_general_memory_complexity}
\end{equation}
With $E=2D$, $S=16$, and $R=2$, this reduces to the linear bound
\begin{equation}
\mathcal M(\mathrm{Vim})=O(BND).
\label{eq:app_vim_memory_complexity}
\end{equation}
For VisionHOPE, an upper bound on intermediate activation memory including token representations, directional token-dependent quantities, and chunk boundary storage is given by
\begin{equation}
\mathcal M(\mathrm{VisionHOPE})=O\!\left(BND+B|\mathcal D|ND_m+Bn_hNd^2\sum_{r\in\mathcal D}\frac{1}{C_r}\right).
\label{eq:app_visionhope_general_memory_complexity}
\end{equation}
The first two terms account for token representations and directional buffers. The third term bounds the storage of $d\times d$ gains, with one gain per chunk for each head and scan direction. Each gain is shared by the five memories within the corresponding SRNL instance. Computing the query projection once does not change the asymptotic storage required for the directional quantities. Substituting the model settings $|\mathcal D|=4$, $D_m=D/2$, $d=16$, and $n_h=D/32$ gives $O(BND+BD(H+W))$ for an $H\times W$ feature map. For square feature maps, this bound further simplifies to
\begin{equation}
\mathcal M(\mathrm{VisionHOPE})=O(BND+BD\sqrt N)=O(BND).
\label{eq:app_visionhope_memory_complexity}
\end{equation}

\paragraph{Resolution-scaling measurements.}
We compare DeiT, Vim, and P-VisionHOPE models at input resolutions from $224\times224$ to $1280\times1280$. All measurements use a single NVIDIA A100 GPU, BF16, and a fixed batch size of $8$. Figure~\ref{fig:visionhope_efficiency_tiny_small} reports the Tiny and Small results, complementing the Base comparison in Figure~\ref{fig:visionhope_efficiency_base}. Across both model scales, P-VisionHOPE maintains linear computation and memory growth with token count. Its efficiency advantage over DeiT widens substantially with resolution, yielding markedly lower FLOPs and memory use together with higher throughput at large inputs. Compared with Vim, P-VisionHOPE maintains a comparable FLOP profile, achieves higher throughput across all tested resolutions, and uses less memory at high resolutions.

\begin{figure}[t]
\centering
\includegraphics[width=0.99\linewidth]{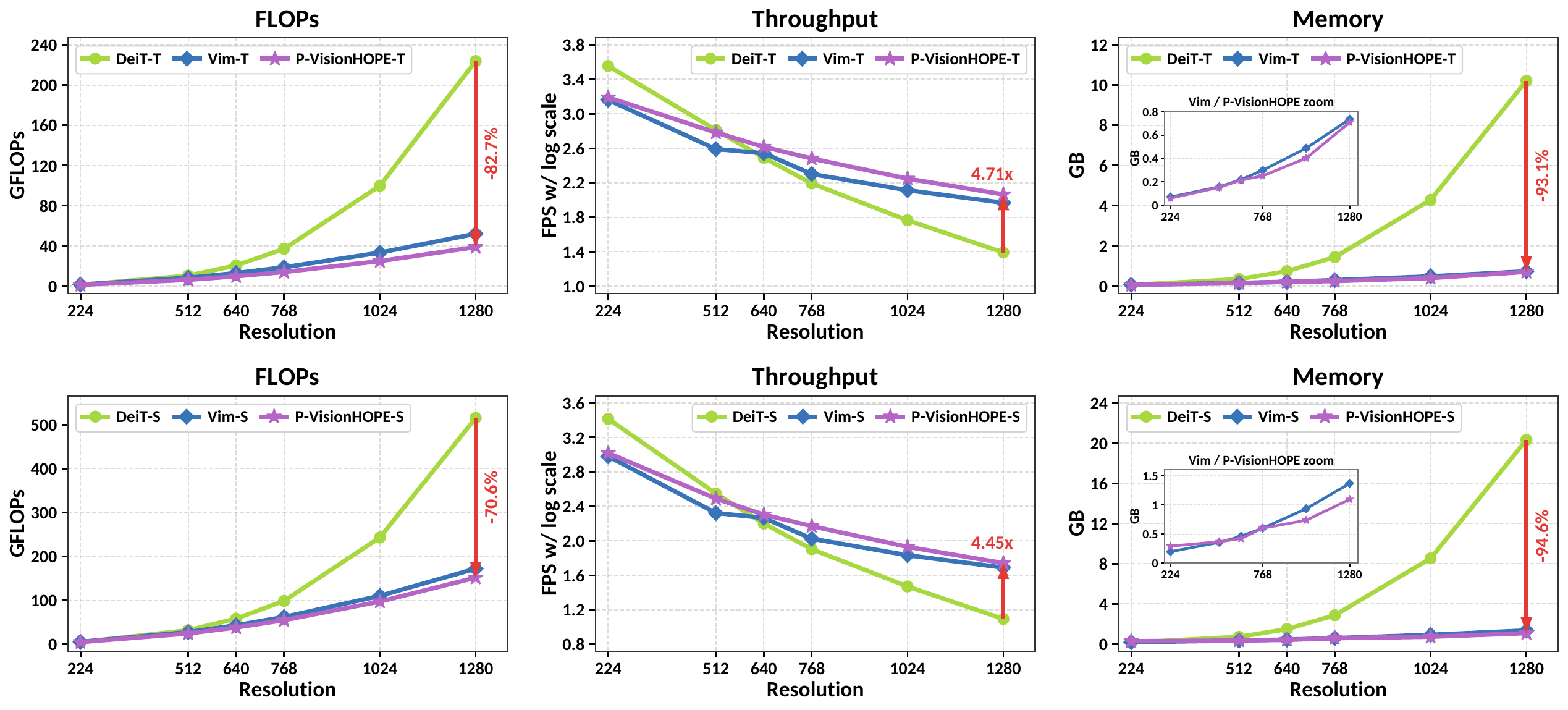}
\caption{Efficiency comparison among DeiT, Vim, and P-VisionHOPE at the Tiny and Small scales. We plot FLOPs, throughput, and memory footprint for input resolutions from $224\times224$ to $1280\times1280$. Hardware measurements use one NVIDIA A100 GPU with BF16 and a fixed batch size of $8$.}
\label{fig:visionhope_efficiency_tiny_small}
\end{figure}

\subsection{COCO Results under the $3\times$ Schedule}
\label{app:coco_3x_results}
We additionally evaluate VisionHOPE on COCO under the longer $3\times$ training schedule to determine whether its transfer performance persists with extended optimization. Table~\ref{tab:coco_maskrcnn_3x} reports complete bounding-box and mask AP metrics, including the corresponding $\mathrm{AP}_{50}$ and $\mathrm{AP}_{75}$ results, for the Tiny and Small backbones. VisionHOPE retains competitive performance at both scales.

\begin{table}[t]
\centering
\caption{COCO object detection and instance segmentation with Mask R-CNN under the $3\times$ schedule. FLOPs and AP are reported in G and \%. VisionHOPE delivers competitive performance.}
\label{tab:coco_maskrcnn_3x}
\scriptsize
\setlength{\tabcolsep}{1.1pt}
\renewcommand{\arraystretch}{1.03}
\belowrulesep=0pt
\aboverulesep=0pt

\begin{minipage}[t]{0.495\linewidth}
\centering
\resizebox{\linewidth}{!}{%
\begin{tabular}{lccccccc}
\toprule
\textbf{Method} & \textbf{FLOPs} & $\mathbf{AP}^{b}$ & $\mathbf{AP}^{b}_{50}$ & $\mathbf{AP}^{b}_{75}$ & $\mathbf{AP}^{m}$ & $\mathbf{AP}^{m}_{50}$ & $\mathbf{AP}^{m}_{75}$ \\
\midrule
ConvNeXt-T & 262 & 46.2 & 67.9 & 50.8 & 41.7 & 65.0 & 44.9 \\
InternImage-T & 270 & 49.1 & 70.4 & 54.1 & 43.7 & 67.3 & 47.3 \\
CSWin-T & 279 & 49.0 & 70.7 & 53.7 & 43.6 & 67.9 & 46.6 \\
MILA-T & 255 & 48.8 & 71.0 & 53.6 & 43.8 & 68.0 & 46.8 \\
VMamba-T & 271 & 48.8 & 70.4 & 53.5 & 43.7 & 67.4 & 47.0 \\
LocalVMamba-T & 291 & 48.7 & 70.1 & 53.0 & 43.4 & 67.0 & 46.4 \\
H-ViT$^{3}$-T & 271 & 48.9 & 71.0 & 53.4 & 44.0 & 68.0 & 47.5 \\
\rowcolor{VisionHOPEGreen}
{VisionHOPE-T} & {266} & {49.4} & {70.6} & {54.5} & {44.1} & {67.8} & {47.8} \\
\bottomrule
\end{tabular}
}
\end{minipage}\hfill
\begin{minipage}[t]{0.495\linewidth}
\centering
\resizebox{\linewidth}{!}{%
\begin{tabular}{lccccccc}
\toprule
\textbf{Method} & \textbf{FLOPs} & $\mathbf{AP}^{b}$ & $\mathbf{AP}^{b}_{50}$ & $\mathbf{AP}^{b}_{75}$ & $\mathbf{AP}^{m}$ & $\mathbf{AP}^{m}_{50}$ & $\mathbf{AP}^{m}_{75}$ \\
\midrule
ConvNeXt-S & 348 & 47.9 & 70.0 & 52.7 & 42.9 & 66.9 & 46.2 \\
InternImage-S & 340 & 49.7 & 71.1 & 54.5 & 44.5 & 68.5 & 47.8 \\
CSWin-S & 342 & 50.0 & 71.3 & 54.7 & 44.5 & 68.4 & 47.7 \\
MILA-S & 319 & 50.5 & 71.8 & 55.2 & 44.9 & 69.1 & 48.2 \\
VMamba-S & 349 & 49.9 & 70.9 & 54.7 & 44.2 & 68.2 & 47.7 \\
LocalVMamba-S & 414 & 49.9 & 70.5 & 54.4 & 44.1 & 67.8 & 47.4 \\
H-ViT$^{3}$-S & 349 & 50.5 & 72.0 & 55.5 & 45.0 & 69.1 & 48.8 \\
\rowcolor{VisionHOPEGreen}
{VisionHOPE-S} & {365} & {50.5} & {71.7} & {55.5} & {45.0} & {69.0} & {48.9} \\
\bottomrule
\end{tabular}
}
\end{minipage}
\end{table}

\subsection{Additional Ablation Studies}
\label{app:additional_ablations}
\paragraph{Step-size control behavior.}
To interpret the results of the step-size control ablations in Table~\ref{tab:visionhope_core_ablations}, we examine the step mappings and clamp activation rates. The soft injection cap is smooth and strictly increasing in the raw step $\eta_t$, so $\bar\eta_t$ preserves the graded modulation produced by the learning-rate memory while approaching the token-dependent limit $\eta_t^{\mathrm{inj}}$. The injection clamp $\min(\eta_t,\eta_t^{\mathrm{inj}})$ is active when $\eta_t/\eta_t^{\mathrm{inj}}>1$, a condition met by $49.68\%$ of the token-wise updates during inference with the P-VisionHOPE-S injection-clamp variant on the ImageNet-1K validation set. When active, it replaces $\eta_t$ with $\eta_t^{\mathrm{inj}}$ and removes direct dependence on raw-step magnitude, plausibly contributing to the lower accuracy. The spectral clamp is active when $\bar\eta_t/\eta_t^{\mathrm{spec}}=\bar\eta_t\lVert k_t\rVert_2^2/(2\alpha_t)>1$. In the full P-VisionHOPE-S model, this condition holds for only $0.0044\%$ of token-wise updates during inference on the same validation set and $0.0011\%$ across the first ten warmup epochs of the standard training schedule. Since the spectral clamp leaves candidate steps below this limit unchanged, it modifies only these rare updates while enforcing the retained-transition bound, consistent with the matching accuracies of the full model and the variant without it. The soft spectral cap instead applies a second smooth map, $\eta_t^{\mathrm{spec}}\bigl(1-\exp(-\bar\eta_t/\eta_t^{\mathrm{spec}})\bigr)$, to $\bar\eta_t$. Although this map respects the retained-transition bound and preserves small steps to first order, it modifies every positive candidate step, including those already within the limit. This additional remapping plausibly explains its lower accuracy. These comparisons help explain our choice of a soft injection cap and a spectral clamp.

\paragraph{Configuration sensitivity.}
We further examine four configuration choices using P-VisionHOPE-S. The ratio $D_m/D$ controls the operator's internal width, chunk length determines the boundary-state update frequency, $\gamma_\eta$ sets the scale of the raw step before step-size control, and $d$ specifies the memory head dimension. Table~\ref{tab:visionhope_parameter_ablations} varies each setting and reports its effect on model size, computation, and ImageNet-1K accuracy. Together, these ablations support the defaults used throughout our experiments, including the use of spatially aligned chunks that span a full row or column of visual tokens. Notably, the default chunk length $C=14$ yields a $2.86\times$ inference speedup over the fully token-indexed setting $C=1$, while maintaining comparable Top-1 accuracy. This gain highlights the practical importance of chunking for efficient recurrent computation in visual backbones.

\begin{table}[t]
\centering
\caption{Configuration ablations on P-VisionHOPE-S. The four panels vary $D_m/D$, chunk length, $\gamma_\eta$, and $d$, respectively. Parameters, FLOPs, and Top-1 accuracy are reported in M, G, and \%.}
\label{tab:visionhope_parameter_ablations}
\footnotesize
\setlength{\tabcolsep}{3.0pt}
\renewcommand{\arraystretch}{1.03}
\belowrulesep=0pt
\aboverulesep=0pt
\resizebox{\linewidth}{!}{%
\begin{tabular}[t]{cccc}
\toprule
\textbf{$D_m/D$} & \textbf{\#P} & \textbf{FLOPs} & \textbf{Acc.} \\
\midrule
$1/4$ & 20 & 4.2 & 81.5 \\
\rowcolor{VisionHOPEGreen}
\textbf{$1/2$} & 22 & 4.7 & 82.3 \\
$1$ & 27 & 6.1 & 82.3 \\
$2$ & 44 & 9.8 & 82.5 \\
\bottomrule
\end{tabular}%
\hspace{4pt}%
\begin{tabular}[t]{cccc}
\toprule
\textbf{Chunk} & \textbf{\#P} & \textbf{FLOPs} & \textbf{Acc.} \\
\midrule
$1$ & 22 & 6.2 & 82.4 \\
$7$ & 22 & 4.8 & 82.2 \\
\rowcolor{VisionHOPEGreen}
\textbf{$14$} & 22 & 4.7 & 82.3 \\
$28$ & 22 & 4.7 & 82.0 \\
\bottomrule
\end{tabular}%
\hspace{4pt}%
\begin{tabular}[t]{cccc}
\toprule
\textbf{$\gamma_\eta$} & \textbf{\#P} & \textbf{FLOPs} & \textbf{Acc.} \\
\midrule
$0.0125$ & 22 & 4.7 & 82.2 \\
\rowcolor{VisionHOPEGreen}
\textbf{$0.025$} & 22 & 4.7 & 82.3 \\
$0.05$ & 22 & 4.7 & 82.1 \\
$0.1$ & 22 & 4.7 & 82.1 \\
\bottomrule
\end{tabular}%
\hspace{4pt}%
\begin{tabular}[t]{cccc}
\toprule
\textbf{$d$} & \textbf{\#P} & \textbf{FLOPs} & \textbf{Acc.} \\
\midrule
$4$ & 22 & 4.5 & 81.8 \\
$8$ & 22 & 4.6 & 82.0 \\
\rowcolor{VisionHOPEGreen}
\textbf{$16$} & 22 & 4.7 & 82.3 \\
$32$ & 22 & 5.3 & 82.2 \\
\bottomrule
\end{tabular}
}
\end{table}

\paragraph{Effect of MESA.}
Table~\ref{tab:visionhope_mesa_ablation} compares hierarchical VisionHOPE models with MILA and H-ViT$^3$ for image classification under training settings with and without MESA. MESA improves Top-1 accuracy for all compared model families. VisionHOPE obtains gains comparable to those of H-ViT$^3$ and achieves higher accuracy than both baselines at every scale under both settings. These results show that its advantage over the compared ViT and TTT models persists without MESA.

\begin{table}[t]
\centering
\caption{ImageNet-1K classification of hierarchical models with and without MESA. The panels correspond to Tiny, Small, and Base models from left to right. The results for MILA and H-ViT$^3$ are taken from their respective papers~\citep{NEURIPS2024_e618724a,Han_2026_CVPR}. Top-1 accuracy is reported in \%.}
\label{tab:visionhope_mesa_ablation}
\footnotesize
\setlength{\tabcolsep}{2pt}
\renewcommand{\arraystretch}{1.03}
\belowrulesep=0pt
\aboverulesep=0pt
\begin{minipage}[t]{0.327\linewidth}
\centering
\begin{tabular}{lc}
\toprule
\textbf{Method} & \textbf{Acc.} \\
\midrule
MILA-T (w/o MESA) & 83.3 \\
MILA-T (w/ MESA) & 83.5 \\
H-ViT$^3$-T (w/o MESA) & 83.5 \\
H-ViT$^3$-T (w/ MESA) & 84.0 \\
\rowcolor{VisionHOPEGreen}
VisionHOPE-T (w/o MESA) & 83.7 \\
\rowcolor{VisionHOPEGreen}
VisionHOPE-T (w/ MESA) & 84.1 \\
\bottomrule
\end{tabular}
\end{minipage}\hfill
\begin{minipage}[t]{0.327\linewidth}
\centering
\begin{tabular}{lc}
\toprule
\textbf{Method} & \textbf{Acc.} \\
\midrule
MILA-S (w/o MESA) & 84.2 \\
MILA-S (w/ MESA) & 84.4 \\
H-ViT$^3$-S (w/o MESA) & 84.4 \\
H-ViT$^3$-S (w/ MESA) & 84.9 \\
\rowcolor{VisionHOPEGreen}
VisionHOPE-S (w/o MESA) & 84.7 \\
\rowcolor{VisionHOPEGreen}
VisionHOPE-S (w/ MESA) & 85.2 \\
\bottomrule
\end{tabular}
\end{minipage}\hfill
\begin{minipage}[t]{0.327\linewidth}
\centering
\begin{tabular}{lc}
\toprule
\textbf{Method} & \textbf{Acc.} \\
\midrule
MILA-B (w/o MESA) & 85.0 \\
MILA-B (w/ MESA) & 85.3 \\
H-ViT$^3$-B (w/o MESA) & 84.9 \\
H-ViT$^3$-B (w/ MESA) & 85.5 \\
\rowcolor{VisionHOPEGreen}
VisionHOPE-B (w/o MESA) & 85.2 \\
\rowcolor{VisionHOPEGreen}
VisionHOPE-B (w/ MESA) & 85.6 \\
\bottomrule
\end{tabular}
\end{minipage}
\end{table}

\begin{figure}[t]
\centering
\includegraphics[width=0.99\linewidth]{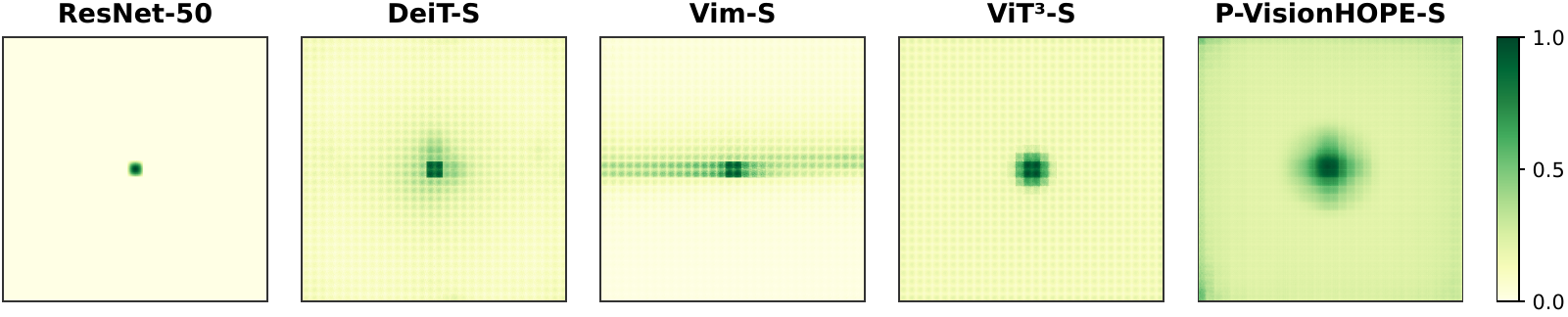}
\caption{Qualitative comparison of global perception through Effective Receptive Fields (ERFs) at $512\times512$ resolution. The maps are computed from the geometrically centered $2\times2$ features at the native output of the second processing block and averaged across the complete ImageNet-1K validation set. Darker colors indicate stronger relative influence on the centered representation.}
\label{fig:visionhope_global_erf}
\end{figure}

\subsection{Visual Analysis}
\label{app:visual_analysis}
\paragraph{Global perception.}
The Effective Receptive Field (ERF) of an output representation measures the input region that influences its response, with each pixel magnitude indicating the strength of that influence~\citep{NIPS2016_c8067ad1}. To compare global perception across visual backbones, we evaluate ImageNet-1K-pretrained models at $512\times512$ resolution and select the geometrically centered $2\times2$ features from the native output of the second processing block. We sum their positive activations, backpropagate the resulting scalar to the input, retain positive input gradients, sum across RGB channels, and average the maps across the full 50,000-image validation set.
Figure~\ref{fig:visionhope_global_erf} shows how strongly different input regions influence the centered representation, with darker colors indicating larger relative responses. ResNet-50~\citep{He_2016_CVPR} exhibits a localized receptive field, DeiT-S and ViT$^3$-S remain more concentrated around the center, and Vim-S presents a strongly directional pattern. In contrast, P-VisionHOPE-S distributes its influence broadly across the two-dimensional input while retaining a clear central response. This indicates stronger global perception in VisionHOPE.

\begin{figure}[t]
\centering
\includegraphics[width=0.99\linewidth]{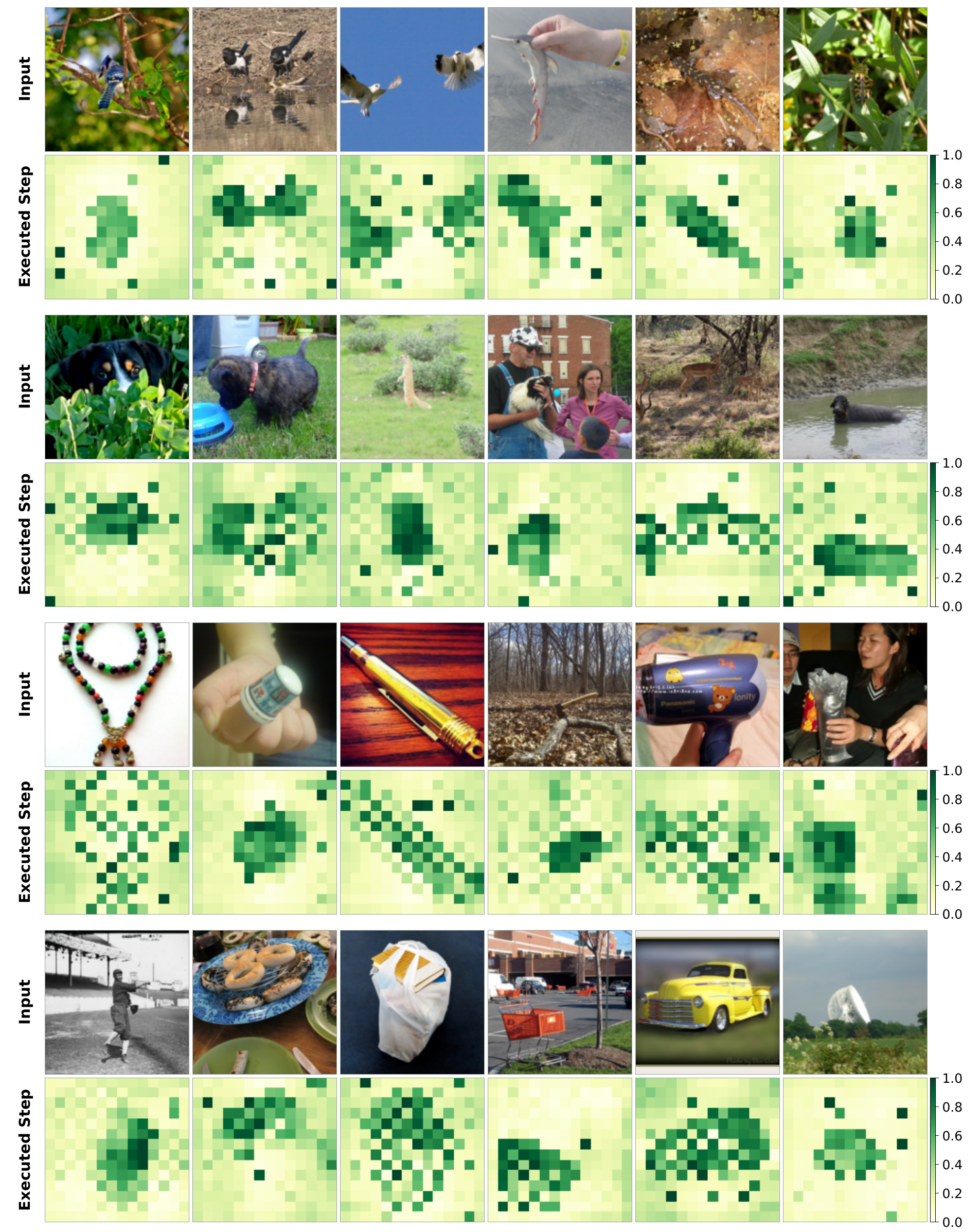}
\caption{Spatial variation of the final executed step in P-VisionHOPE-S across several ImageNet-1K validation images. Each two-row block pairs input images with the corresponding $14\times14$ maps of $\widetilde\eta_t$ from the final VisionHOPE block, restored to image coordinates and averaged across heads and four scan directions. Darker colors indicate relatively larger executed steps within each image.}
\label{fig:visionhope_executed_step_maps}
\end{figure}

\paragraph{Within-image learning dynamics.}
To examine how VisionHOPE modulates its within-image learning dynamics across spatial locations, we visualize the final executed step $\widetilde\eta_t$ from Equation~(\ref{eq:visionhope_spectral_clamp}) in the final block of P-VisionHOPE-S. For each ImageNet-1K validation image, the token-wise values from the four directional scan orders are restored to the common $14\times14$ feature grid and averaged across heads and directions. 
Figure~\ref{fig:visionhope_executed_step_maps} shows that the executed step forms structured, content-dependent patterns across spatial locations. Darker regions represent relatively larger values of $\widetilde\eta_t$ and therefore a stronger tendency to modify the coupled memories at the corresponding locations. These regions frequently align with prominent objects and distinctive visual structures, suggesting that VisionHOPE learns semantically meaningful spatial variation in its within-image learning dynamics. The maps thus provide a direct and intuitive view of this adaptive process.

\section{Strengths \& Weaknesses}
\label{app:strengths_weaknesses}
\paragraph{Strengths.}
VisionHOPE unifies within-image adaptation of stored content, update representations, and learning dynamics through five coupled memories. Its stability-matched step-size control provides non-expansion guarantees for both token-wise and chunk-wise memory recurrences. The operator supports hierarchical and plain backbones with computation and activation memory that scale linearly with visual token count. Results on classification and dense prediction benchmarks demonstrate its effectiveness across tasks, while ablations further support the main design choices.

\paragraph{Weaknesses and future work.}
Our current evaluation focuses on general-purpose visual recognition. We have not yet evaluated VisionHOPE as a visual encoder in Multimodal Large Language Models (MLLMs). Its effectiveness on specialized downstream tasks in medical imaging and remote sensing also remains to be empirically validated. We plan to extend VisionHOPE to these settings and assess its performance in both multimodal understanding and domain-specific visual tasks.

\end{document}